%% file: main.tex
\documentclass[runningheads]{llncs}
\usepackage[T1]{fontenc}
\usepackage{hyperref}       % hyperlinks
\usepackage[sort&compress,square, numbers]{natbib}
\usepackage{graphicx}
\usepackage{bbold}
\usepackage{amsfonts}
\usepackage{amsmath}
\usepackage[ruled]{algorithm2e}
\usepackage{multirow}
\usepackage{caption}
\usepackage{booktabs}
\usepackage{subcaption}
\usepackage{wrapfig}
\usepackage[dvipsnames]{xcolor}
\usepackage{tikz}
\usepackage{tikz-cd}

\usepackage{color}

\begin{document}
\title{Leveraging Human Feedback for Semantically-Relevant Skill Discovery}
\author{Maxence Hussonnois\inst{1}\orcidID{0000-0002-0721-1850} \and
Thommen George Karimpanal\inst{2}\orcidID{0000-0001-8918-3314} \and
Santu Rana\inst{1}\orcidID{0000-0003-2247-850X}}
\authorrunning{M. Hussonnois et al.}
\institute{$A^2 I^2$, Deakin University, Geelong, Australia \\ 
\email{{m.hussonnois@deakin.edu.au}, {santu.rana@deakin.edu.au}} 
\and
School of IT, Deakin University, Geelong, Australia. \\
\email{thommen.karimpanalgeorge@deakin.edu.au}}
\maketitle              % typeset the header of the contribution
\begin{abstract}
\input{sections/1_abstract}

\end{abstract}

\input{sections/2_introduction}
\input{sections/3_related_work}

\input{sections/4_preliminaries}

\input{sections/4_methods}
\input{sections/5_experiments}
\input{sections/6_limitations}
\input{sections/6_conclusion}
\bibliographystyle{splncs04}
\bibliography{sample.bib}

\newpage
\appendix
\input{sections/appendix}

\end{document}

%% file: sections/1_abstract.tex
Unsupervised skill discovery in reinforcement learning aims to intrinsically motivate agents to discover diverse and useful behaviours. However, unconstrained approaches can produce unsafe, unethical, or misaligned behaviours. To mitigate these risks and improve the practical desireability of discovered skills, recent work grounds the discovery process by leveraging human preference feedback. However, preference-based approaches are feedback-inefficient and inherently ill-equipped to deal with skill spaces composed of a variety of different skills such as running, jumping, walking, etc. To overcome this limitation, we introduce semantic labelling, a novel and feedback-efficient approach that leverages human cognitive strengths to identify and label semantically meaningful behaviours. Based on semantic labelling, we propose \textit{Semantically Relevant Skill Discovery (SRSD)}\footnote{Link to a \href{https://github.com/HussonnoisMaxence/SRSD}{github repository}}, a novel human-in-the-loop approach that collects semantic labels from human feedback and learns a reward function to encourage skills to be more semantically diverse and relevant. Through our experiments in a 2D navigation environment and four locomotion environments, we demonstrate that SRSD can improve semantic diversity and discover relevant behaviours while scaling effectively to a large variety of behaviours.

%% file: sections/2_introduction.tex
\section{Introduction}
Given a well-designed reward function, deep reinforcement learning~\citep{sutton2018reinforcement} can effectively solve a variety of decision-making problems like video games~\citep{Mnih2015}, board games~\citep{silver2016mastering}, and robotic tasks~\citep{Lillicrap2016}. However, designing reward functions that enable the training of general-purpose agents capable of exhibiting a diverse range of behaviours remains a fundamental open problem. In contrast, animals naturally learn a repertoire of behaviours, often called skills, which are both reusable and can be combined to solve new problems. Inspired by this observation, unsupervised skill discovery~\citep{eysenbach2018, Sharma2020Dynamics-Aware, Campos2020ExploreDA, park2021lipschitz} aims to intrinsically motivate agents to autonomously learn diverse sets of skills, achieving notable success in domains such as locomotion and manipulation. 
\begin{figure}[htbp]
    \centering
    \includegraphics[width=1\textwidth]{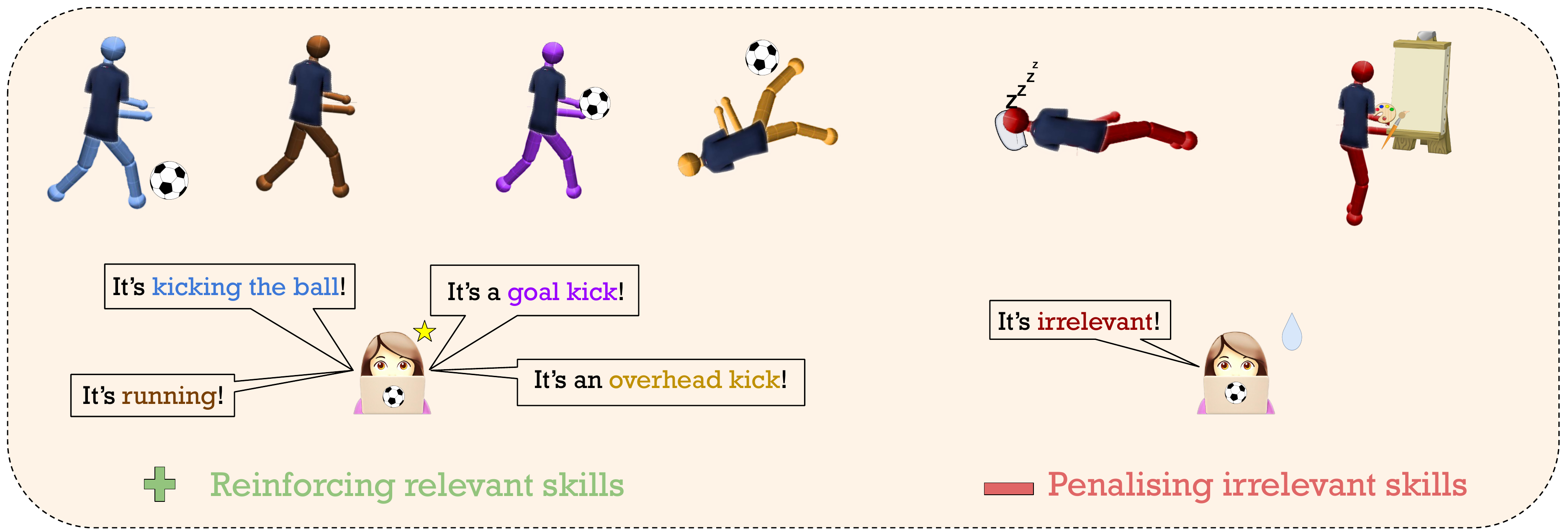} 
    \caption{It is challenging to learn many types of skills which are all contextually relevant. For instance, football needs many skills such as running and kicking the ball. While sleeping and painting may be distinct, possible activities irrelevant to the current context (playing football). By querying humans to label semantically relevant skills ("running" or "kicking the ball") while labelling other behaviours as irrelevant, our method discovers diverse and contextually relevant skills.}
    \label{fig: introduction|example}
\end{figure}

Despite promising results, these approaches fail at consistently learning desirable skills and often learn misaligned, unsafe or irrelevant behaviours~\citep{kim2023safeskill,hussonnois2023controlled}. To increase the desirability of the skill learned, recent approaches suggest guiding skill discovery by initialising training with fixed prior knowledge in the form of expert demonstrations~\citep{wan2024lotus}, videos~\citep{xu2023xskill}, safety constraints~\citep{kim2023safeskill}, or prompts for language models~\citep{rho2024language}. Unfortunately, these approaches are unable to adapt to unforeseen or changing environments. This underscores the need for methods that proactively incorporate human guidance during training, enabling agents to continuously adjust to evolving situations.

Recent works~\citep{hussonnois2023controlled, hussonnois2025human} proposed involving humans in the skill discovery loop by learning a reward function from human preferences. However, the preference feedback process becomes increasingly inefficient as the number of different types of preferences increases ~\citep{zhao2016learning}. This limitation is particularly problematic for skill discovery, which naturally involves a spectrum of distinct behaviours, such as running and kicking the ball (see Figure~\ref{fig: introduction|example}). Thus, these methods often result in skill sets that, although aligned, lack semantic diversity (such as only discovering diverse ways of running forward, while ignoring other semantically relevant skills). Alternatively, humans possess an innate cognitive ability to intuitively group observed behaviours into meaningful, “basic-level” categories, such as kicking the ball, goal-kicking, or running. This phenomenon is extensively documented in cognitive psychology~\citep{Rosch1988-ROSPOC}, which shows that people can naturally assign behaviours to such categories without explicit training, thereby providing maximal information with minimal cognitive effort. Importantly, this categorisation process remains efficient even as the number of skill types increases. Thus, leveraging human semantic categorisation ability into the skill discovery loop holds promise for both enhancing semantic diversity among discovered skills and achieving greater feedback efficiency compared to traditional preference-based methods.

In this work, we leverage human abilities in semantic categorisation to improve both semantic diversity, that is, the range of skill types discovered, and feedback efficiency within the skill discovery process. To achieve this, we introduce a semantic labelling procedure (see Figure~\ref{fig: introduction|example}): for each trajectory, the human labeller first indicates whether any relevant behaviour is present; if not, the trajectory is marked as 'irrelevant'. Otherwise, the labeller specifies the most appropriate semantic category. We provide both theoretical and empirical evidence that this semantic labelling approach scales more effectively than traditional preference-based feedback as the number of skill types increases, making it especially suited to learning a diverse set of behaviours. Building on this, we present \textit{Semantically Relevant Skill Discovery (SRSD)}, a framework that collects semantic labels from human feedback, trains a semantic predictor, and introduces a reward function to explicitly encourage the learning of semantically meaningful skills. We show that skill discovery methods augmented with our reward formulation achieve superior results on DeepMind Control Suite benchmarks. Furthermore, we demonstrate that incorporating a distributional representation of critics further boosts the performance when learning from multiple intrinsic reward functions.
In summary, the main contributions are:

\begin{itemize}
  \item \emph{Semantic Labelling}, a novel feedback-efficient mechanism that learns a semantic relevance predictor from human feedback.
  \item \emph{Semantically Relevant Skill Discovery (SRSD)} a novel framework that leverages semantic labels to discover more semantically diverse skills.
  \item Mitigating aleatoric uncertainty when learning with multiple intrinsic reward functions by incorporating a distributional perspective.
  \item Qualitative and quantitative evaluation of SRSD, with suitable comparisons with existing baselines for learning diverse and fine-tuned skills.
\end{itemize}

%% file: sections/3_related_work.tex
%% Illustration of the method
\begin{figure*}[htbp]
    \centering
    \includegraphics[width=1\textwidth]{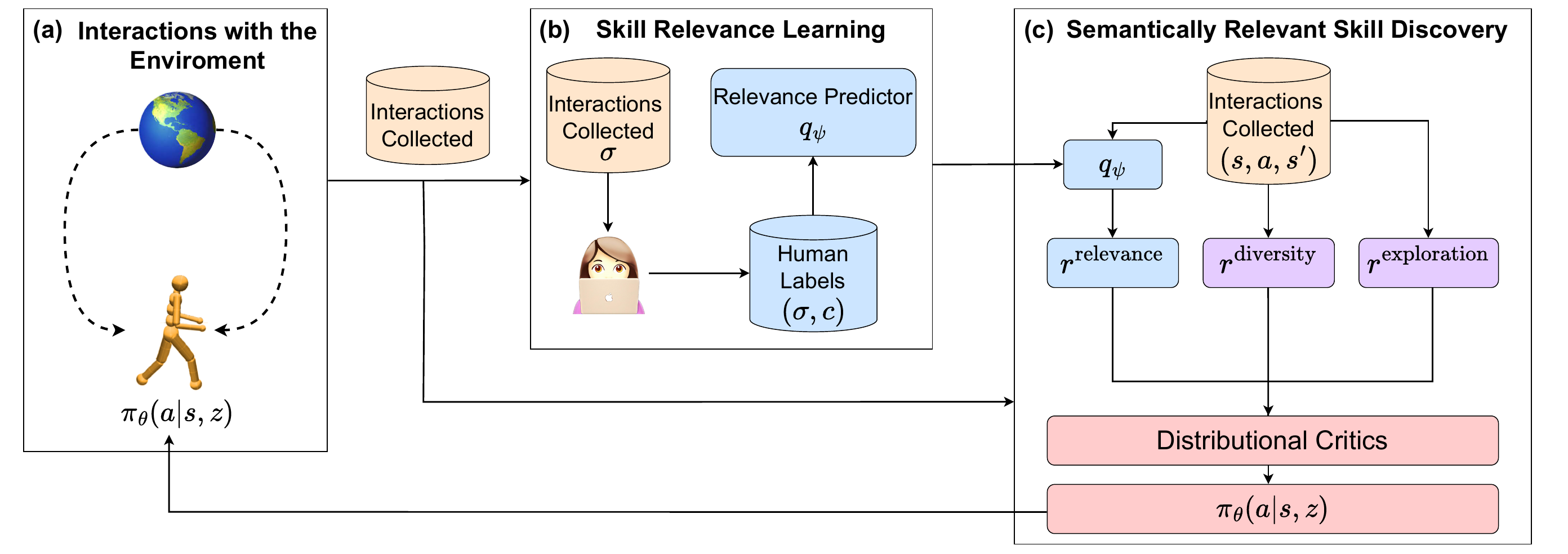} 
    \caption{
  Illustration of our method. First (a), SRSD explores potential relevant type of behaviours by maximising skill diversity and exploring with $r^{\text{div}}$ and $r^{\text{exploration}}$. Then (b), we start eliciting  semantic labels, and learn a semantic predictor from these labels. This semantic predictor is used to additionally reward skills to be more semantically relevant with $r^{\text{relevance}}$ (c).}
\label{fig: illustration|srsd}
\end{figure*}

\section{Related Work}
\subsubsection{Unsupervised Skill Discovery.} Unsupervised skill discovery aims to simultaneously explore the space of learnable behaviours and learn useful and diverse skills \citep{SUTTON1999181}. In this regard, previous works~\cite{eysenbach2018, Gregor2016, achiam2018} relied on the reverse mutual information between states and skills to train skills to be distinguishable. Recent works~\cite{Campos2020ExploreDA, laskincontrastive} improve the exploration capability by using the forward form of the mutual information with contrastive representation learning~\citep{chen2020simple, pmlr-v9-gutmann10a}. Alternative works \cite{park2021lipschitz, park2023controllability, park2023metra} suggest distance-maximising objectives to learn dynamic skills in locomotion and pixel-based environments. Yet, these methods do not consider the environment's underlying context, and tend to produce dangerous or unaligned skills. To this end, the safety-aware skill discovery framework~\citep{kim2023safeskill} was proposed to include user-pre-defined constraints. Alternatively, recent work~\citep{rho2024language, kim2024s} examines how to incorporate context into skill discovery using LLMs. However, these methods rely on user pre-training-based domain knowledge to actively guide the process towards more desirable skills. Our approach instead focuses on reactive guidance of skills, incorporating the human context only when requested, leading to a more realistic skill discovery process.

\subsubsection{Preference-based Skill Discovery.} To discover more desirable skills, preference-based skill discovery approaches are based on preference-based RL (PbRL)~\citep{Christiano2017, 2021pebble, park2022surf}. Alternatively, CDP~\citep{hussonnois2023controlled} and HaSD~\citep{hussonnois2025human} use preferences with human-in-the-loop~\citep{2021pebble} methods to discover more desirable and aligned skills. Our work also employs a human-in-the-loop approach but introduces semantic labelling. Semantic labelling focuses on minimal requirements for behaviour classification, making it more effective than preferences in environments with numerous semantic behaviours like running, walking and jumping.

%% file: sections/4_preliminaries.tex
\section{Preliminaries}
\label{subsec: prelim|setting}
\subsubsection{Skill Discovery Settings.} We consider a Markov Decision Process (MDP) without a reward function, defined as \begin{math}\mathcal M = (\mathcal S, \mathcal A, \mathcal P) \end{math}, where \begin{math}\mathcal S\end{math},\begin{math}\mathcal A\end{math} and \begin{math}\mathcal P\end{math} respectively denote the state spaces, action spaces and transition dynamics. Consistent with prior work \citep{eysenbach2018, Campos2020ExploreDA,park2023controllability}, we note random variables for states and actions as $S$ and $A$. We then consider a latent variable $z \in Z$, on which we condition our policy. We refer to a policy conditioned on a fixed $z$, with parameter $\theta$, as a skill and it is denoted as \begin{math}\pi_\theta(a|s,z)\end{math}. In this work, we follow Contrastive Skill Discovery methods \citep{laskincontrastive, liu2023comsd} that train $\pi_\theta(a|s,z)$ by maximising the forward form of the mutual information $I(S;Z)$ between $S$ and $Z$:
\begin{equation}
\begin{aligned}
\label{eq: mi_objective}
   I(S;Z) = \mathcal H(S) - \mathcal H (S|Z),
\end{aligned}
\end{equation}
where $\mathcal H(\cdot)$ denotes the Shannon entropy. Contrastive Skill Discovery methods \citep{laskincontrastive, liu2023comsd} focus on the forward form of mutual information to explicitly optimise state entropy $\mathcal H(S)$ to improve exploration. We adopt this method for skill discovery due to its state of the art performance on URLB \citep{laskin2021urlb}.

\subsection{Contrastive Skill Discovery}
\label{subsubsec: CSD|CRL}
%\paragraph{\textbf{Contrastive Skill Discovery}}
Contrastive skill discovery methods leverage contrastive representation learning to estimate condition entropy and particle-based exploration to estimate state entropy. In the following sections, we will detail both components. 

\subsubsection{Contrastive Representation Learning.} Contrastive Skill Discovery methods use contrastive representation learning to approximate conditioned state entropy $\mathcal H (S|Z) $ with a discriminator $q_\phi$ capable of handling high continuous skill dimensions, improving skill diversity. To this end, we maximise the following lower bound for the conditional entropy (see \citep{liu2023comsd} for a detailed derivation): 

\begin{equation}
\begin{aligned}
\label{eq: lower_bound}
  -\mathcal H (S|Z) \geq \mathbb E [\log q_\phi(s,s',z)]  \quad\text{with}\quad q_\phi(s,s',z) = \frac{g_{\phi_1}(s,s')^\top g_{\phi_2}(z)}{\left\lVert g_{\phi_1}(s,s') \right\rVert \left\lVert g_{\phi_2}(z) \right\rVert T},
\end{aligned}
\end{equation}

where $q_\phi$ is a parameterised function used to estimate $p(s,s',z)$. More precisely, $q_\phi$ is defined with a similar inner product proposed in~\citep{chen2020simple}. $T$ is a temperature parameter and $g_{\psi_{1,2}}$ are neural networks used to learn a state and a skill embedding. Then both encoders $g_{\phi_{1,2}}$ are trained to to maximise similarity between state transitions and corresponding skills by optimising the following Noise Contrastive Estimation (NCE) loss \cite{pmlr-v9-gutmann10a}:
\begin{eqnarray}
\label{eq: nce_loss}
\mathcal{L}_{NCE}= \log q_\phi(s,s',z) -\log \frac{1}{N}\sum^N_{j=1}\exp(q_\phi(s_j,s_j',z))),  
\end{eqnarray}
where $s,s'$ are sampled using skill $z$ and $s_j, s_j'$ are sampled using $N-1$ distinct skills. Finally, the trained discriminators can be used to intrinsically reward skills \begin{math}\pi_\theta(a|s,z)\end{math} to generate diverse behaviours for different random variables $z \in Z$ with the following reward $r^{\text{div}} = q_{\phi}(s, s'|z)$.

\subsubsection{Particle-Based Exploration.} As mentioned earlier in this section, optimising skills \begin{math}\pi_\theta(a|s,z)\end{math} solely with $r^{\text{div}}$ is not enough to ensure exploration of the behaviour space. To mitigate this issue, CIC proposes rewarding skills for maximising state entropy by estimating the state entropy $\mathcal H(S)$ with a particle-based entropy estimator, APT~\cite{liu2021behavior}. Specifically, APT estimates the state entropy by computing the Euclidean distance between the embedding of a state and all its $k$-nearest neighbours resulting in the following reward:
\begin{eqnarray}
\label{eq: apt_particle estimation}
r^{\text{exp}} = \mathcal H_{APT}(S) \propto \frac{1}{k}\sum_{h^* \in N_k} \log \left\lVert h-h^*  \right\rVert ,
\end{eqnarray}
where $h_i$ is an embedding of $(s, s')$ from $g_{\phi_1}$ and $h^*$ is the $k$NN embedding.%, $k$ is the number of nearest neighbours.

%% file: sections/4_methods.tex
\section{Semantically-Relevant Skill Discovery (SRSD)}
To improve semantic diversity (diversity in the type of skills) in skill discovery, we consider discovering diverse skills that have a high chance of being categorised among relevant semantics by a human overseer. First, we describe the problem formulation for discovering diverse skills that fall into relevant semantics in Section \ref{subsec: srsd|setting}. We then introduce our novel semantic labelling feedback mechanism and the resulting relevance reward in Section \ref{subsec: method|human_feedback}. Finally, in Section \ref{subsec: method|SRSD} we present our overall SRSD method which is illustrated in Figure~\ref{fig: illustration|srsd} and described in Algorithm~\ref{alg: SRSD}.

\subsection{Semantically-Relevant Skill Discovery Settings}
\label{subsec: srsd|setting}
We extend the skill discovery setting introduced in Section \ref{subsec: prelim|setting} to incorporate the semantic variable $c \in \mathcal{C}$, which denote different type of behaviour. Semantics $\mathcal{C}$ are composed of multiple relevant semantics (relevant type of behaviour) $\mathcal{C}^+$ \footnote{In this work we consider a fixed number of relevant classes throughout the pre-training.} and one irrelevant semantic $\mathcal{C}^-$ (corresponding to all irrelevant types of behaviour). Our goal in this work is to learn skills \begin{math}\pi_\theta(a|s,z)\end{math} that are diverse and that belong to relevant semantics $\mathcal{C}^+$.

To this end we proposed to follow a multi-objective optimisation problem, where the objectives are the skill discovery objective presented in Section \ref{subsec: prelim|setting} and our novel relevance objective.
We formulate our relevance objective as maximising the probability that trajectories $\tau$ generated by $\pi_\theta(a|s,z)$ being labelled as one of the relevant semantics $\mathcal{C}^+$. Thus we propose to train \begin{math}\pi_\theta(a|s,z)\end{math} to maximise the following multi-objective problem: 
\begin{eqnarray}
\label{eq: srsd_objective}
\max_{\pi_{\theta}} (I(S;Z), p(c\in \mathcal{C}^+|\tau) ),
\end{eqnarray}
where the resulting skills are diverse, with a high chance of being classified among a relevant semantics. Next, we describe how to train \begin{math}\pi_\theta(a|s,z)\end{math} to maximise our novel relevance objective.

\subsection{Learning Semantically-Relevant Behaviour from Human Feedback}
\label{subsec: method|human_feedback}
To maximise the probability of $\tau$ induced by $\pi_\theta(a|s,z)$ to be labelled as one of the relevant semantics $\mathcal{C}^+$, it is necessary to access the semantic distribution across trajectory $p(c|\tau)$. However such distributions cannot be observed or obtained from the environment. Thus we propose to learn it through human feedback by introducing "Semantic Labelling", a novel human-feedback mechanism.

\subsubsection{Semantic Labelling Process.} Semantic labelling involves incorporating a human-in-the-loop approach, where a human observer categorises trajectory segments as belonging to either relevant or irrelevant semantics. Specifically, during semantic labelling, the human is provided with a trajectory segment of horizon $H$, such as $\sigma = ((s_k, a_k), \ldots, (s_{k+H}, a_{k+H})) \in (\mathcal{S} \times \mathcal{A})^H$ and assigns it a label $c \in \mathcal{C}$. Recognising cognitive effort and time as a relevant constraint for humans, we assume a fixed budget $B$ of human feedback, represented as $\zeta^B = \{\sigma_i, c_i\}^B_{i=0}$ during the training process. %We provide more detail of the human feedback process in Section \ref{subsubsec: method|queries}.

\subsubsection{Scaling with respect to $|\mathcal{C}|$: Semantic vs Preference Feedback.} To show that semantic-labelling feedback is more suitable than preference to handle multiple semantics, we first define preference over multiple semantics. As pair-wise comparison (preferences) cannot identify multiple semantics alone \citep{zhao2016learning}, we sample preference feedback over multiple semantics by first selecting a relevant semantic, and then sampling preference feedback. Then we analyse for each feedback mode the probability of sampling a query that will provide information about a relevant semantic $c\in C^+$. We describe that the order for each mode on $|\mathcal{C}|$ is: 
\begin{proposition}\label{proposition: scaling_c}
Let $|\mathcal{C}| \geq 3$ denote the total number of semantics. Then, the probability that a single preference feedback query is informative about a relevant semantic is $O\left(\frac{1}{|\mathcal{C}|}\right)$, whereas the probability of a semantic-label query remains $O(1)$, independent of $|\mathcal{C}|$.
\end{proposition} 
The proof is provided in Appendix \ref{subsec: apdx|proof}. This proposition explains that semantic-label feedback scale well to the number of semantics as it is independent of $|\mathcal{C}|$. In contrast, the probability of preference feedback being informative of a relevant semantic will diminish as $|\mathcal{C}|$ increases.
 \begin{figure*}[htbp]
    \includegraphics[width=1.0\textwidth, height=0.125\textheight]{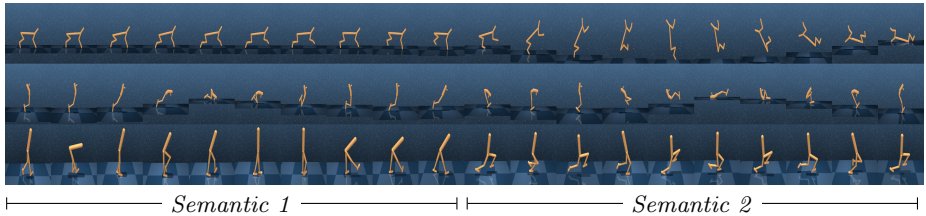} 
    \begin{tikzpicture}
          \coordinate (A) at (-6.5,0);
          \coordinate (B) at ( -0.5,0);
        \draw[|-|] (A) -- (B) node[midway,fill=white] {\emph{Semantic 1}};
    \end{tikzpicture}
    \begin{tikzpicture}
          \coordinate (A) at (-0.5,0);
          \coordinate (B) at (5.5,0);
         \draw[|-|] (A) -- (B) node[midway,fill=white] {\emph{Semantic 2}};
    \end{tikzpicture}
    \caption{Qualitative visualisations of skill discovered with Cheetah, Hopper and Walker. Cheetah learns to run forward and backflip while Hopper learns to hop forward and backflip and finally Walker learns to balance itself by walking and running.}
    \label{fig: experiments|dmc|qualitative}
\end{figure*}

\subsubsection{Semantic Predictor.} We model the semantic distribution across trajectory $p(c|\tau)$ with a parametric semantic predictor $p_\psi(c|\sigma)$. As we assume all behaviours to follow the Markov property, we define the semantic predictor with the scoring function $q_\psi$ in Equation \ref{eq: relevance_predictor}:
\begin{equation}
\label{eq: relevance_predictor}
 p_{\psi} (c|\sigma)= \frac{\exp(\sum_t^H q_\psi(c|s_t,a_t))}{\sum_{j=0}^{|\mathcal{C}|}\exp(\sum_t^H q_\psi( c_j|s_t,a_t))}.
\end{equation}

Intuitively, this means that the probability of a trajectory segment $\sigma$ being predicted as $c$ depends exponentially on the sum over the segment of an underlying score function. The scoring function $q_\psi$ is modelled as a neural network with parameters $\psi$ and is trained with a cross-entropy loss. By minimising this loss function, the scoring function learns to generate high scores for the corresponding relevant or irrelevant semantics. 

\subsubsection{Collecting Human Feedback and Active Sampling.} During training, we repeatedly query a human overseer to collect feedback up to a maximum budget. Since some semantics are rare or emerge only in the later stages, a major challenge is ensuring a fair representation of all relevant classes in the dataset while avoiding over-representation of irrelevant examples. To address this, in each query session, we apply a filtering process by dividing the batch of segments among relevant semantics using our semantic predictor. We then sample uniformly among relevant semantics until a predetermined quota is met. This avoids oversampling and enables a more balanced and efficient use of human feedback. We detail the process in the Algorithm~\ref{alg: Active_Sampling} in the Appendix.

\subsubsection{Semantically Relevant Rewards Function.} Finally, using the scoring function $q_\psi$ previously introduced, we propose the following novel semantically relevant reward function:
\begin{eqnarray}
\label{eq: reward_z_c_2}
r^{\text{rel}} = \log \frac{\exp( q_\psi (f(c|z)|s,a))}{\sum_{j=0}^{|C|} \exp(q_\psi (c_j|s,a))}  \quad \text{with} \quad f:Z \to \mathcal{C}^+,
\end{eqnarray}
that rewards skills for generating trajectories that have a high chance of being predicted as relevant. To avoid skills collapsing into relevant semantics that are easier to learn, we propose to divide the skill space $Z$ equally among the $|\mathcal{C}^+|$ semantics by replacing the last $|\mathcal{C}^+|$ dimension of the $z$ latent variable with a one-hot vector that specifies a relevant semantic. Here, $f(c|z)$, is a simple mapping function from $Z$ to $\mathcal{C}^+$ that reads the one-hot segment of to determine the relevant class associated. We provide more details in Appendix~\ref{subsec: apdx|implementation_details}.

\subsection{Semantically Relevant Skill Discovery (SRSD)}
\label{subsec: method|SRSD}
In this section, we present Semantically Relevant Skill Discovery (SRSD), our overall method that learns diverse skills that belong to relevant semantics by maximising the Equation~\ref{eq: srsd_objective}. 
First, SRSD ensures that skills explore potential relevant semantics by maximising skill diversity while exploring with $r^{\text{exp}}$ and $r^{\text{div}}$ introduced in Section~\ref{subsubsec: CSD|CRL}. Then, we elicit semantic labels from human feedback, using which we learn a semantic predictor. We then use this semantic predictor with $r_{relevance}$ to additionally reward skills to be more semantically relevant. In other words, we train \begin{math}\pi_\theta(a|s,z)\end{math} to maximise the linear combination of $r^{\text{exp}},r^{\text{div}}, r^{\text{rel}}$ such that:
\begin{equation}
\begin{aligned}
\label{eq: RDSD}
   r^{\text{SRSD}}_t &=  r^{\text{exp}}_t +  r^{\text{div}}_t + w_t \cdot  r^{\text{rel}}_t  \quad\text{with}\quad   w^r_t  =
    \begin{cases}
      0& \text{if \begin{math} t \leq t_{0} \end{math}} \\
      \frac{\zeta}{|B|} & \text{otherwise}
    \end{cases}   
\end{aligned}
\end{equation}

Here, $w_t $ is a dynamic weight initialised at $0$ and increases as we collect feedback from humans until the number of elements in dataset $\zeta$ reaches the budget $B$ as follow:

\subsubsection{Distributional Critic.} SRSD linearly combines three learned reward functions which result in a noisy and more complex reward distribution, making approximating the value function challenging. Thus we propose to adopt a distributional representation of critics~\citep{bellemare2017distributional} which allows for learning the aleatoric uncertainty of our reward functions. This is done by incorporating Truncated Quantile Critics (TQC)~\citep{kuznetsov2020controlling} on top of TD3~\citep{pmlr-v80-fujimoto18a} as a backbone reinforcement learning algorithm. TQC integrates the distributional perspective by decomposing the expected return into the atoms of a distributional representation. Additionally, it controls overestimation by truncating the right tail of the reward distribution.

%% file: sections/5_experiments.tex
\section{Experiments}
\label{sec: experiments}
Our experiments aim to address the following questions: (Q1) How does SRSD's performance across various pre-trainings compare to previous methods? (Q2) Is semantic feedback more effective than preference feedback at handling a large range of semantics? 
\begin{figure*}[htbp]
    \centering
    \begin{subfigure}[b]{0.30\textwidth}
        \centering
        \includegraphics[width=1.0\textwidth]{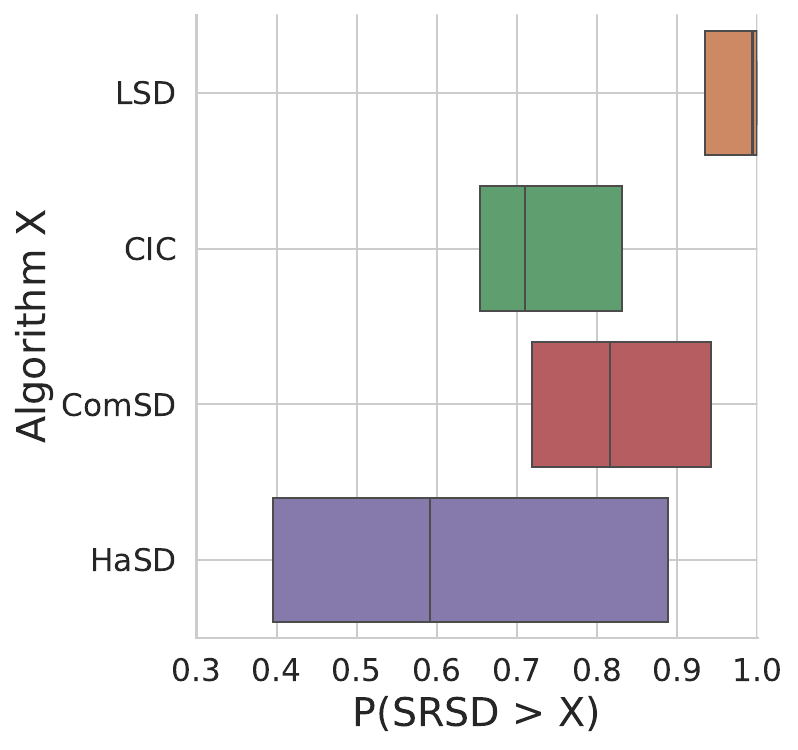} 
        \caption{Zero-shot}
         \label{fig: downstream_task|zero_shot|general}
    \end{subfigure}
    \hfill
    \begin{subfigure}[b]{0.30\textwidth}
        \centering
        \includegraphics[width=1.0\textwidth]{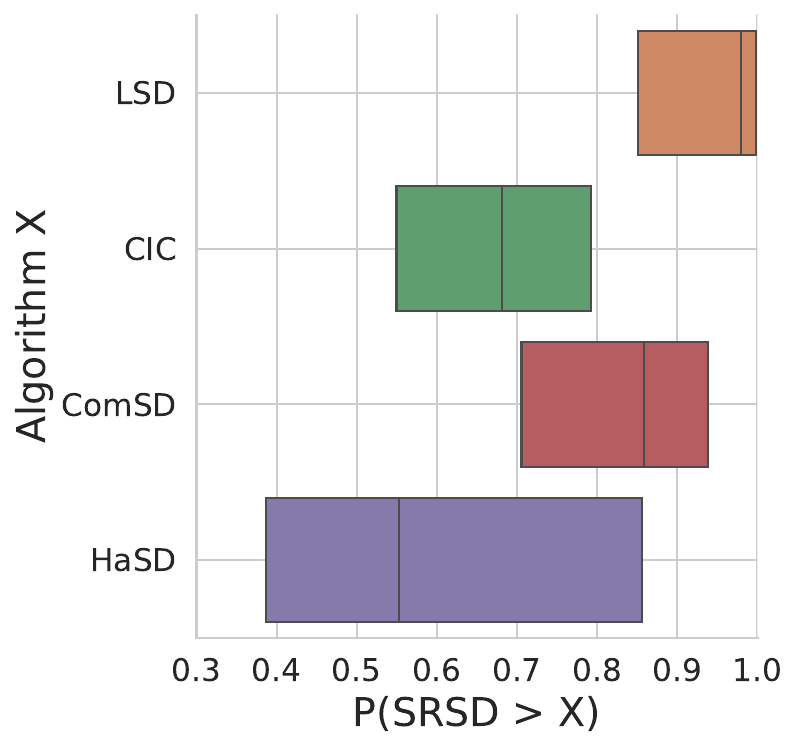} 
        \caption{Few-shot}
        \label{fig: downstream_task|few_shot|general}
    \end{subfigure}
    \hfill
    \begin{subfigure}[b]{0.30\textwidth}
        \centering
        \includegraphics[width=1.0\textwidth]{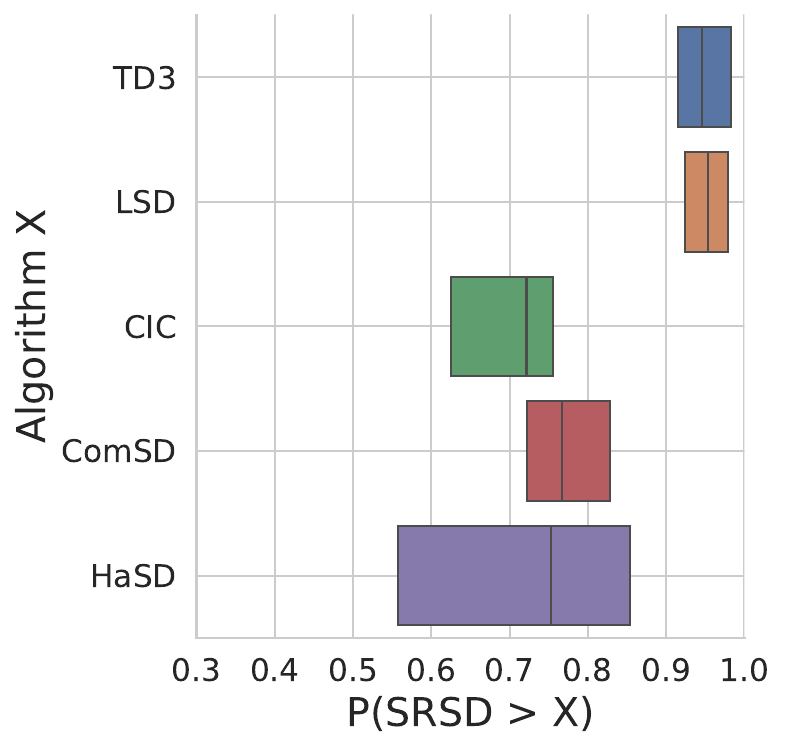} 
        \caption{Fine-tuning}
         \label{fig: downstream_task|fine_tuning|general}
    \end{subfigure}
    \caption{Average SRSD's probability of improvement over baselines on downstream tasks on URLB for Zero-shot(a), Few-short(b) and Fine-tuning evaluation(c). Probability of improvement shows how likely it is for SRSD to outperform algorithm X on a randomly selected task~\cite{agarwal2021deep}. SRSD provides consistency performance gains over all baselines.}
    \label{fig: Experiment|dmc|results|general}
\end{figure*}

\subsection{Experiment Setup}
\label{subsec: experiment|simulating_human_feedback}
\subsubsection{Baselines.} We compare SRSD with the following baselines: TD3~\citep{pmlr-v80-fujimoto18a}, LSD~\citep{park2021lipschitz}, CIC~\citep{laskincontrastive}, ComSD~\citep{liu2023comsd}, and HaSD~\citep{hussonnois2025human}. TD3 serves as the backbone RL algorithm for all methods, providing a baseline without pre-training. LSD, CIC and ComSD acts as the baselines for unsupervised skill discovery without human feedback. HaSD, which uses human preferences, serves as the baseline for human-guided skill discovery. To ensure fairness, we integrated each method with TQC, and extended HaSD to handle multiple preference modes.

\subsubsection{Simulating Human Feedback.} To facilitate experimentation and reproducibility, we simulated human feedback as previous works~\citep{2021pebble}. Human feedback is modelled by computing the probability of each relevant semantic being labelled independently. We select the semantic with the highest probability and use it to stochastically choose between labelling the semantic as relevant/ irrelevant. We provide more details in Appendix~\ref{subsec: apdx|simulating_human_feedback} along with evaluations using simulated human irrationality.

\subsubsection{Environments.} We evaluated our approach across 4 domains from URLB~\cite{laskin2021urlb} and DeepMind Control Suite (DMC)~\cite{tassa2018deepmind} (Walker, Quadruped, Cheetah, Hopper) . Each domain comports 4 downstream tasks to evaluate and fine-tune our skills. In these environments, we define relevant skills as downstream tasks and use downstream task rewards as ground-truth rewards. Additionally, we use a 2D point agent navigation environment in a two-dimensional circular room. In this environment, relevant semantics refers to travelling to one of the desired sectors of the environment. Ground-truth rewards are binary values that reward if the agent is in the desired sector.

\subsubsection{Evaluation.} Following prior work, we pre-trained CoMSD, HaSD, and SRSD for 1M (Nav2d) and 2M (DMC) steps. All methods were initialised with 5 different random seeds. HaSD and SRSD used 1400 feedback instances. In DMC domains, we assess performance by fine-tuning the pre-trained skills of each method on downstream tasks as per URLB~\cite{laskin2021urlb}. In this procedure, 100K environment steps are considered, and during the first 4K steps, we randomly sample skills and select the best-performing ones, which are fine-tuned with extrinsic rewards for 96K steps. For HaSD and SRSD, we sampled only semantically relevant skills. After fine-tuning we evaluated each seed for each method 10 times resulting in 50 scores for each task. We also assessed zero-shot (by evaluating a randomly sampled skill) and few-shot (by evaluating the best-performing skill before being fine-tuned) performances. %capabilities: zero-shot , and few-shot .
\begin{figure*}[htbp]
    \centering
    \includegraphics[width=1.0\textwidth]{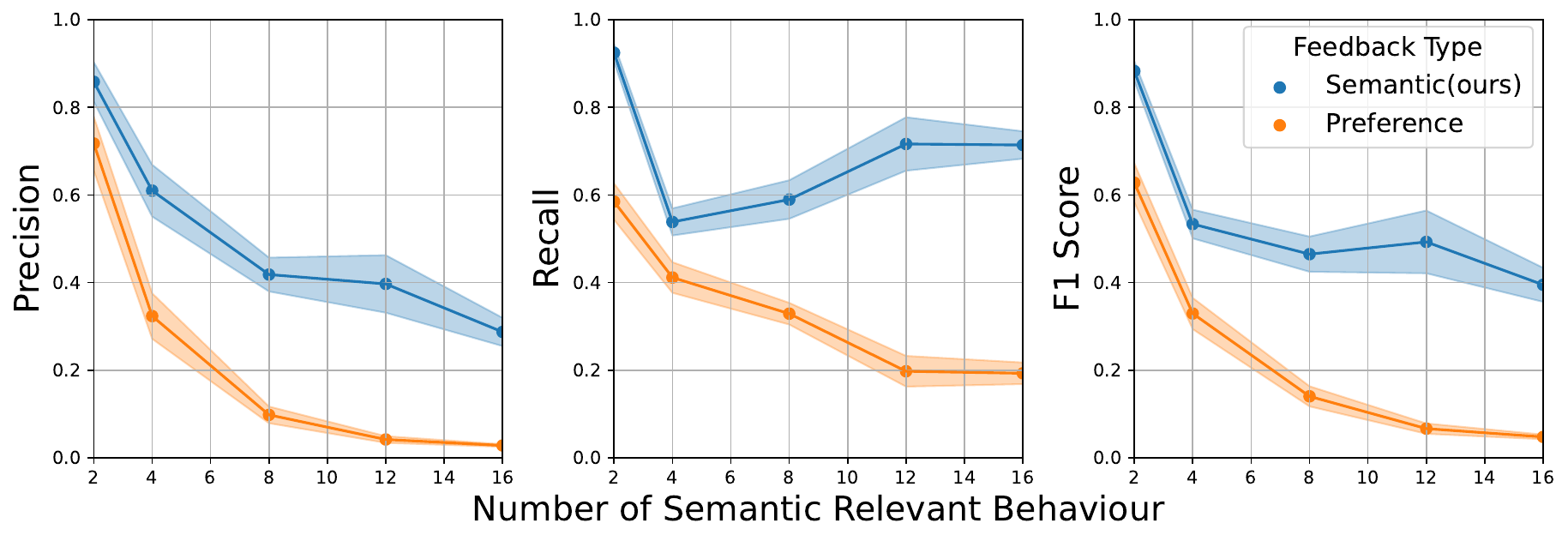} 
    \caption{Impact of semantic and preference feedback on skills learned when handling numerous semantic (2,4,8,12,16) in a 2D navigation environment. Precision measures how many skills are semantically relevant, while recall measures how much of the relevant semantics are covered by skills. This shows that compared to semantic-labels, preference feedback is less efficient in settings with multiple semantics, resulting in poor performance.}
    \label{fig: experiment|nav2d|coverage_metrics}
\end{figure*}

\subsubsection{(Q1) Performance on DMC Suite.} To assess the performance gains brought by SRSD over baseline methods, we report, in Figure~\ref{fig: Experiment|dmc|results|general}, the average probability of improvement as presented in~\citep{agarwal2021deep} across fine-tuning, zero-shot, and few-shot evaluations, respectively. Our analysis reveals that SRSD generally achieves a statistically significant probability of improvement over all baselines (i.e., the probability of improvement exceeds 0.5 and the associated confidence intervals do not overlap 0.5). The main exception arises in comparisons with HaSD on zero-shot and few-shot evaluations, where the increased stochasticity inherent to these scenarios reduces the margin of superiority. Nevertheless, SRSD still demonstrates a mean probability of improvement above 0.5 in these cases, indicating consistent performance gains. We provide a granular perspective with detailed score distributions for all agents and tasks in Appendix~\ref{fig: downstream_task|fine_tuning|tasks} showing SRSD's superior performance on the more challenging tasks. Generally, both SRSD and HaSD achieve high overall performance due to their ability to systematically reduce the skill search space and target semantically relevant behaviours during sampling. Finally, Figure~\ref{fig: experiments|dmc|qualitative} presents qualitative examples of diverse skills discovered by the SRSD framework.

%Additional breakdowns of zero-shot and few-shot performances are provided in Appendix~\ref{subsubsec: appendix|dmc|performances}. 

\subsubsection{(Q2) Preference feedback vs Semantic feedback.} We analyse the impact of semantic and preference feedback on skills learned when handling numerous relevant semantics (2,4,8,12,16) in the 2D navigation environment. To this end, we sampled 1000 skills after pre-training and evaluated the coverage of each semantic by its associated semantic-conditioned skills. In Figure~\ref{fig: experiment|nav2d|coverage_metrics} we compared both feedback types against precision, recall and F1-score metrics. Precision measures how many skills are semantically relevant, while recall measures how much of the relevant semantics are covered by skills. This shows that compared to semantic-labels, preference feedback is less efficient in settings with multiple semantics, resulting in poor performance. Qualitative visualisations of skills for all environments and methods are reported in Appendix~\ref{subsec: appendix|qualitative_results_nav2d}.

\subsubsection{Other Experiments.} We also investigated the following aspects, which are covered in Appendix~\ref{subsec: appendix|additional_results}. Feedback sensitivity: SRSD remains largely robust to the feedback budget; while zero-shot performance decreases as labels drop (40, 100, 400, 1400 instances), it still outperforms ComSD even with only 40 labels.  Active sampling: our strategy allocates the majority of labels to relevant semantics (>50\% in all domains; up to 79\% in Walker), which reduces oversampling of irrelevant semantics, improves fairness across relevant classes (via Jain’s index), and yields measurable performance gains. Distributional perspective: integrating a TQC-based distributional critic delivers statistically significant improvements by managing aleatoric uncertainty and mitigating value overestimation. Reward ablations: all components contribute positively (probability of improvement > 0.5) where $r^\text{div}$ and $r^\text{rel}$ show consistent, significant gains. Real human-in-the-loop: Authors familiar with the task provided 400 feedback signals across 12 query sessions (1.5 minutes each; $\approx$20 minutes total) to learn skills that closely match the simulated-feedback setting.

%% file: sections/6_limitations.tex
%\section{Limitations}

%% file: sections/6_conclusion.tex
\section{Discussion and Limitation}
We introduce Semantic Relevant Skills Discovery (SRSD), a framework using semantic feedback to improve diversity in the type of skills discovered. In four locomotion environments, SRSD effectively discovers diverse and relevant skills. We provide theoretical and empirical evidence that semantic-labels better suit learning different types of skills compared to preference feedback. Future work will extend SRSD to scenarios where relevant semantics are not fixed but evolve dynamically, leading to settings that better mirror real-world skill discovery. A key limitation of our work is that we only learn skills that are semantically meaningful to humans. However, in the context of learning aligned behaviours, this is a desirable characteristic. In addition, while we designed our framework and experiments to closely simulate realistic humans with stochasticity and irrationality, incorporating actual human feedback may introduce additional challenges. We leave this aspect for future work.% Future work should directly incorporate and account for human inconsistencies during skill discovery.

%% file: sections/appendix.tex
\section{Appendix}
\input{appendix/additional_experiment}

\input{appendix/proof}

\input{appendix/implementation_details}
\input{appendix/metrics}
\input{figures/algo}
\input{figures/algo_sampling}
\input{appendix/human_feedback}

%% file: appendix/additional_experiment.tex
\subsection{Additional Experiments}
\label{subsec: appendix|additional_results}

In this section we provide additional results supporting Q1 and Q2 in Section~\ref{sec: experiments} and additional experiments to address the following questions:(Q3) How sensitive is SRSD to feedback budget? (Q4) Does active sampling ensure evenly distributed dataset of semantic labels? (Q5) Does incorporating a distributional perspective improve SRSD performance? (Q6) How does each reward term individually affect training performance? (Q7) How does SRSD perform with real Human-in-the-Loop Experiment?

\subsubsection{(Q1) How does SRSD's performance across various pre-trainings compare to previous methods?}
To provide a more granular perspective, we include detailed score distributions for all agents and tasks in Figure~\ref{fig: downstream_task|fine_tuning|tasks} for the fine-tuning evaluation, focusing on selected baselines. These plots illustrate that SRSD consistently outperforms all baselines, with particularly notable improvements on the more challenging tasks (shown to the right of the plots), such as Quadruped Run, Walker Run, Cheetah Run (forward and backward), and Hopper Hop (forward and backward). 

\begin{figure}[htbp]
    \centering
    \includegraphics[width=1.0\textwidth]{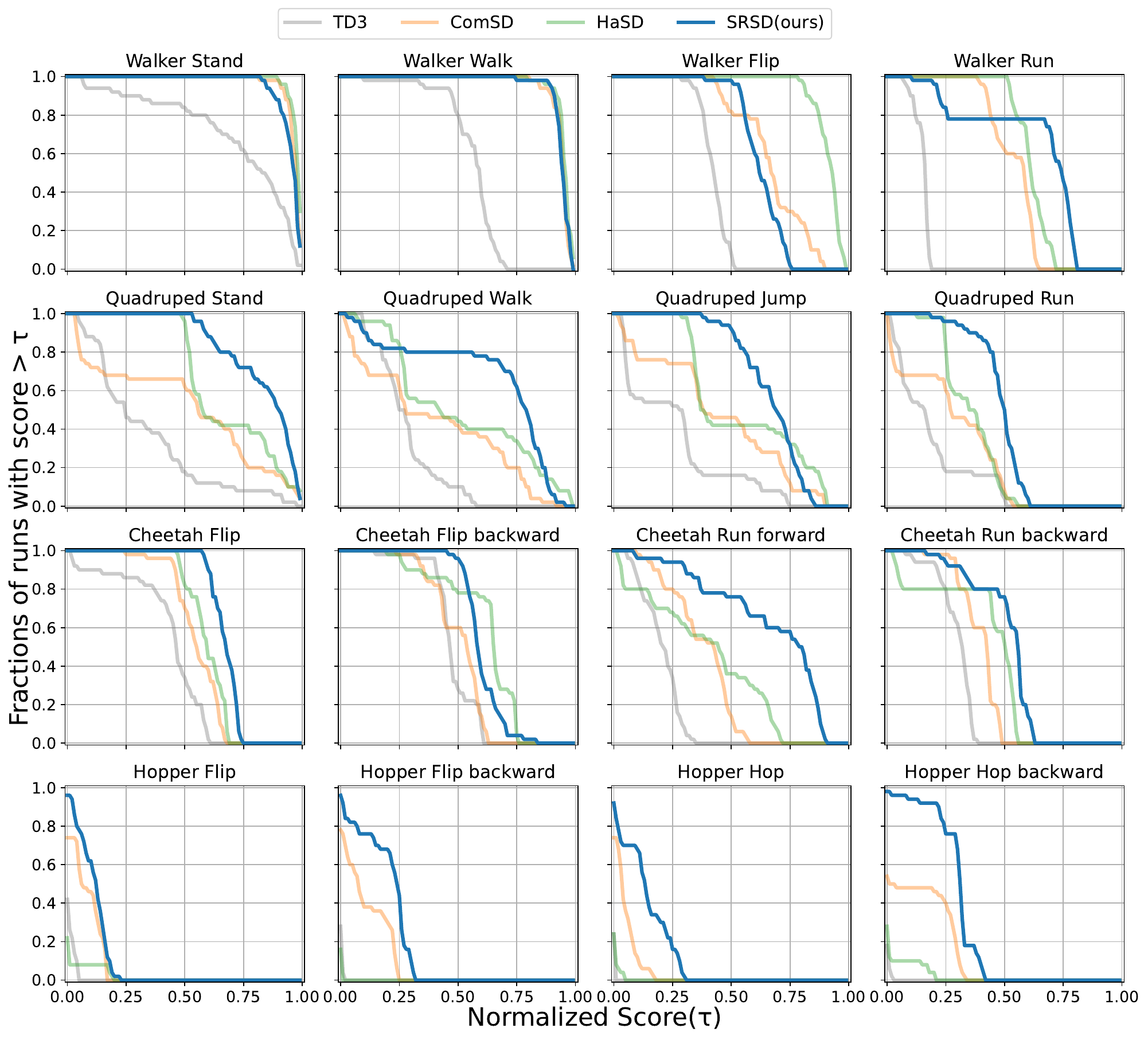} 
    \caption{Score distributions after fine-tuning skills over 12 downstream tasks on URLB. Score distributions show the fraction of runs above a certain score (higher curve is better) \cite{agarwal2021deep}. SRSD achieves competitive performances across all tasks and agents.}
    \label{fig: downstream_task|fine_tuning|tasks}
\end{figure}

\subsubsection{(Q2) Is semantic feedback more effective than preference feedback at handling a large range of semantics?}
\label{subsec: appendix|qualitative_results_nav2d}
In this section, we qualitatively compare skills learned from semantic labels with those learned from preference feedback. Figure~\ref{fig: appendix|qualitative|nav2d|example} shows the target semantic sectors for the 2D navigation environments with 2, 4, 8, 12, and 16 semantics. Figures~\ref{fig: appendix|qualitative|nav2d|hasd} and \ref{fig: appendix|qualitative|nav2d|srsd} display the skill sets learned by HaSD and SRSD, respectively. We can observe that HaSD’s ability to produce skills that cover all the semantic sectors degrades significantly starting at 4 semantics, whereas SRSD learns a skill set that covers all semantics evenly with more consistency.

\begin{figure}[htbp]
    \centering
    \includegraphics[width=1.0\textwidth]{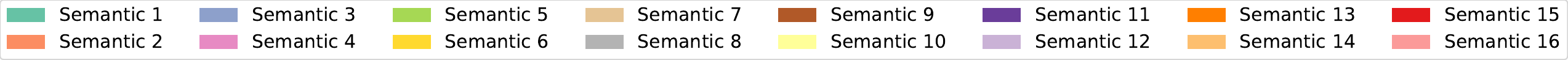}
    \begin{subfigure}[b]{1.0\textwidth}
        \centering
        \includegraphics[width=0.19\textwidth]{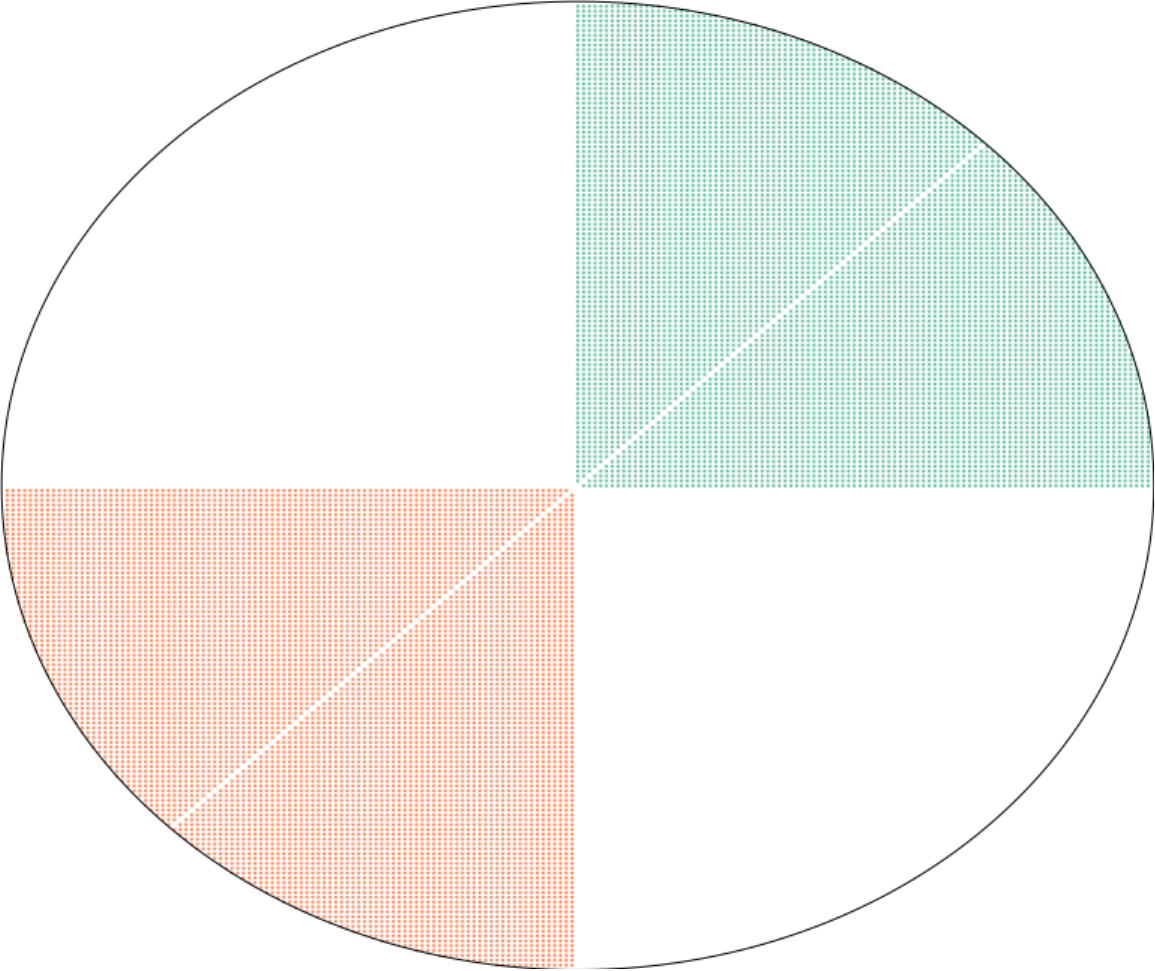} 
        \includegraphics[width=0.19\textwidth]{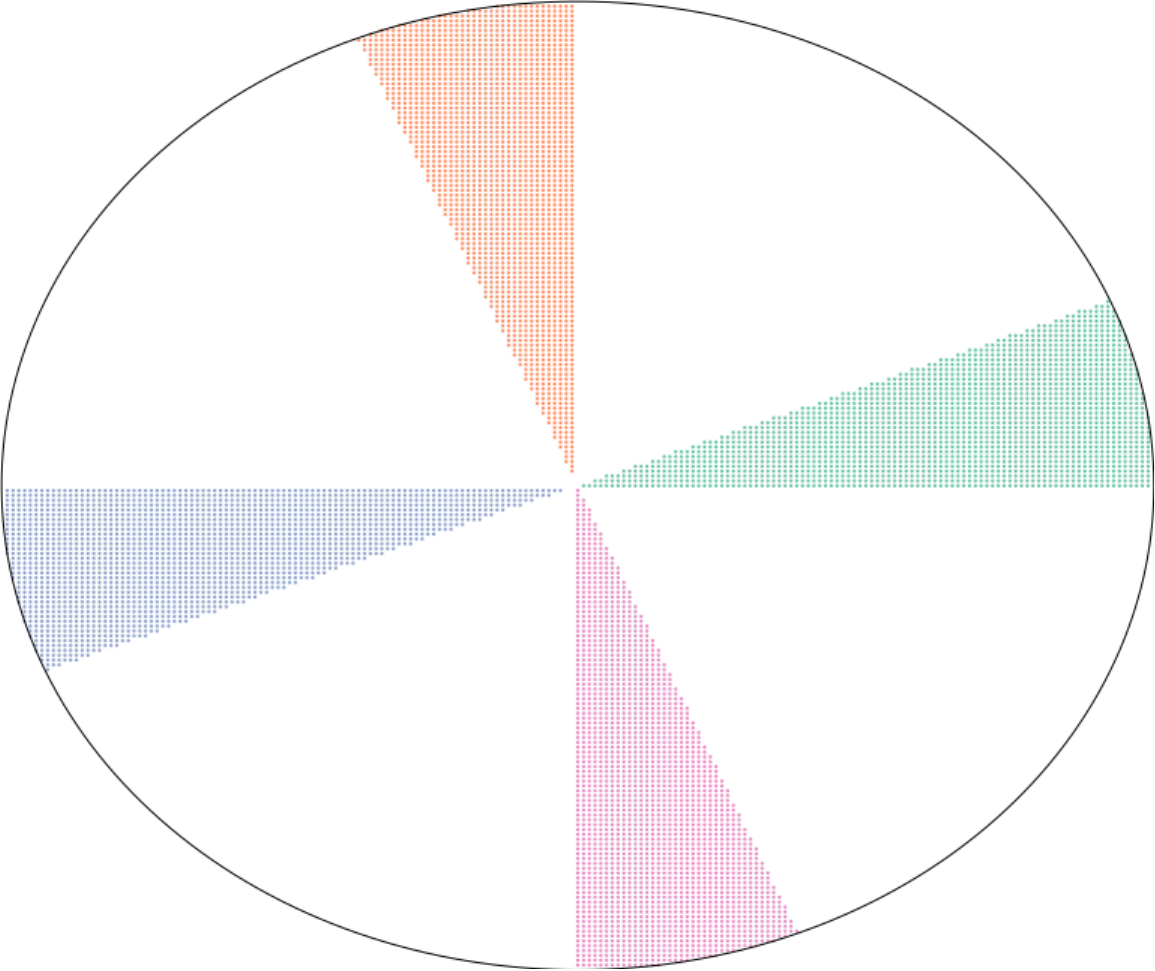} 
        \includegraphics[width=0.19\textwidth]{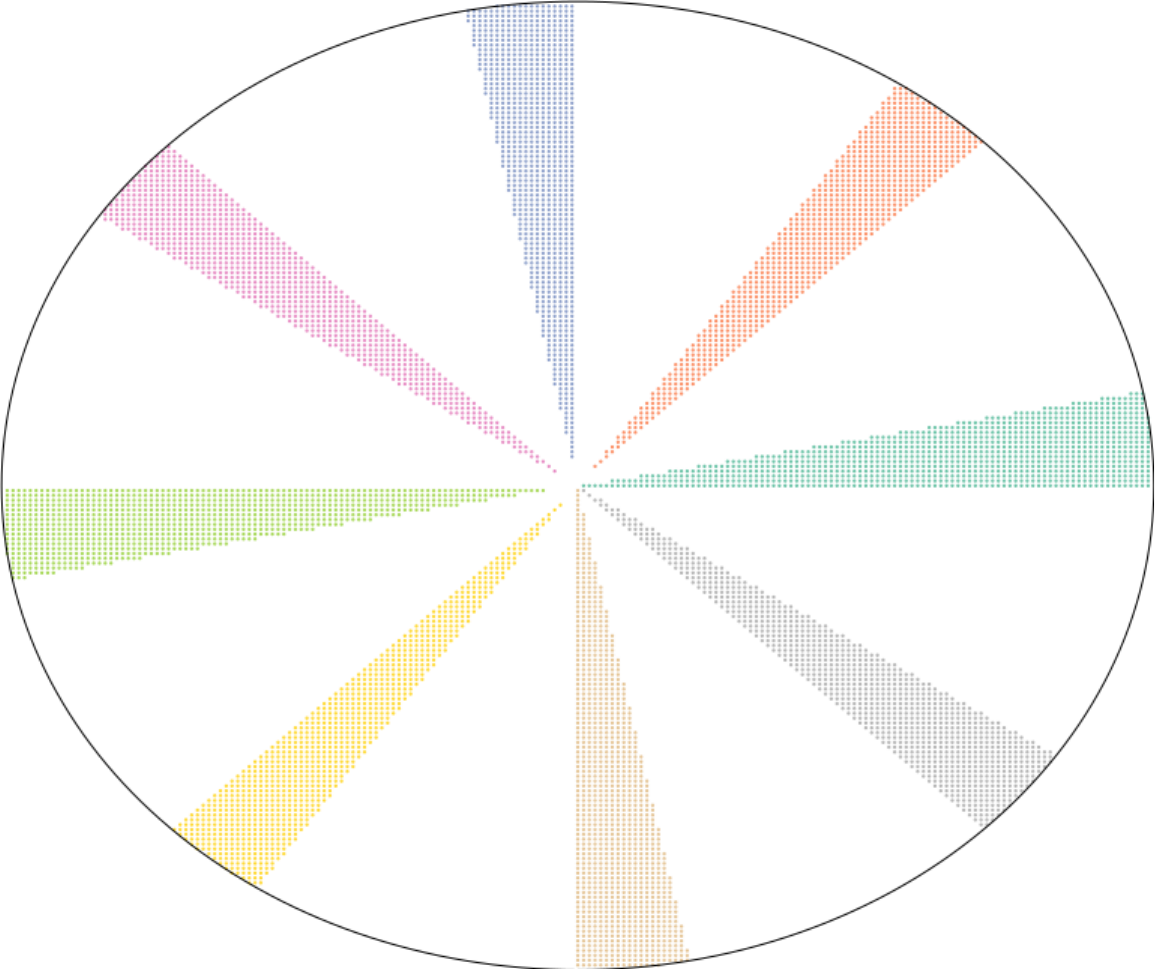}
        \includegraphics[width=0.19\textwidth]{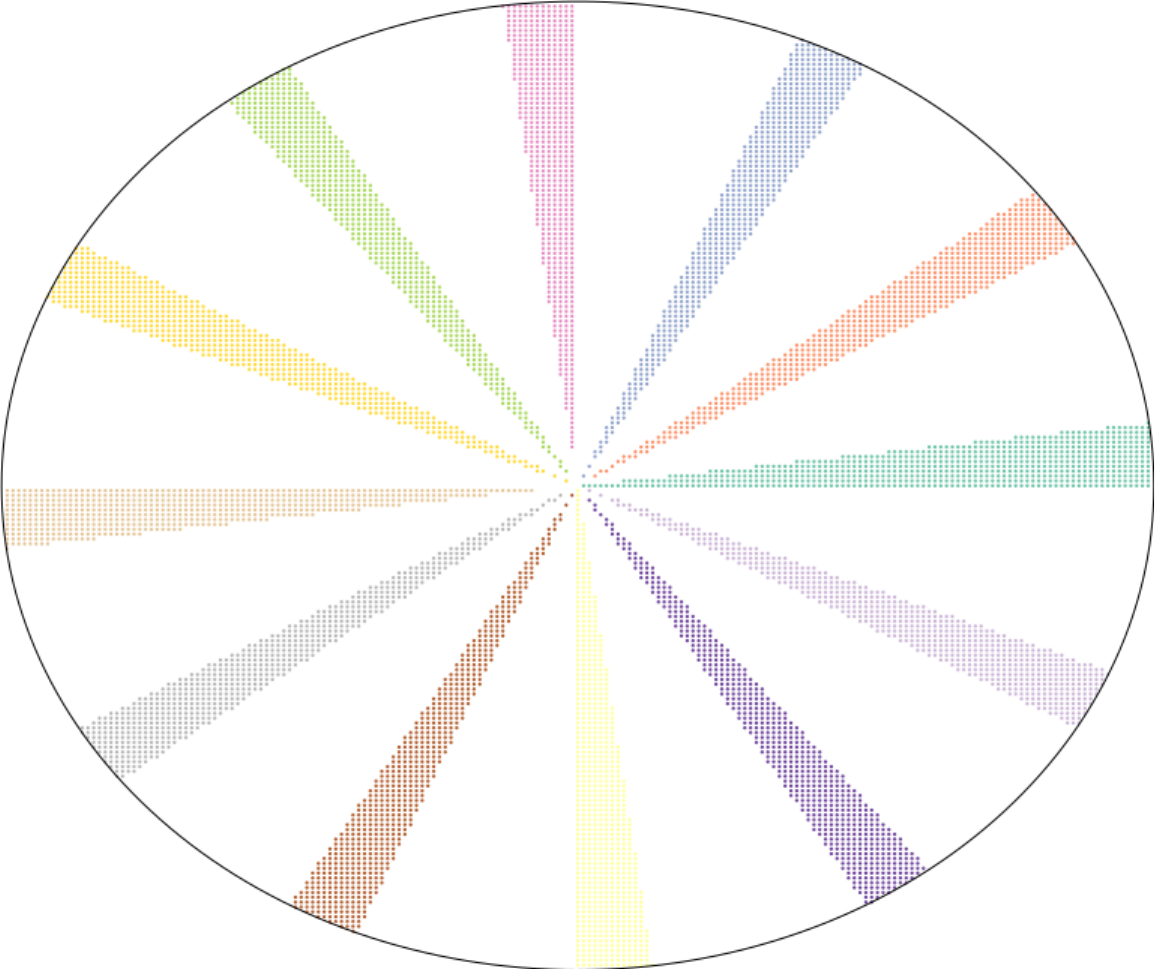} 
        \includegraphics[width=0.19\textwidth]{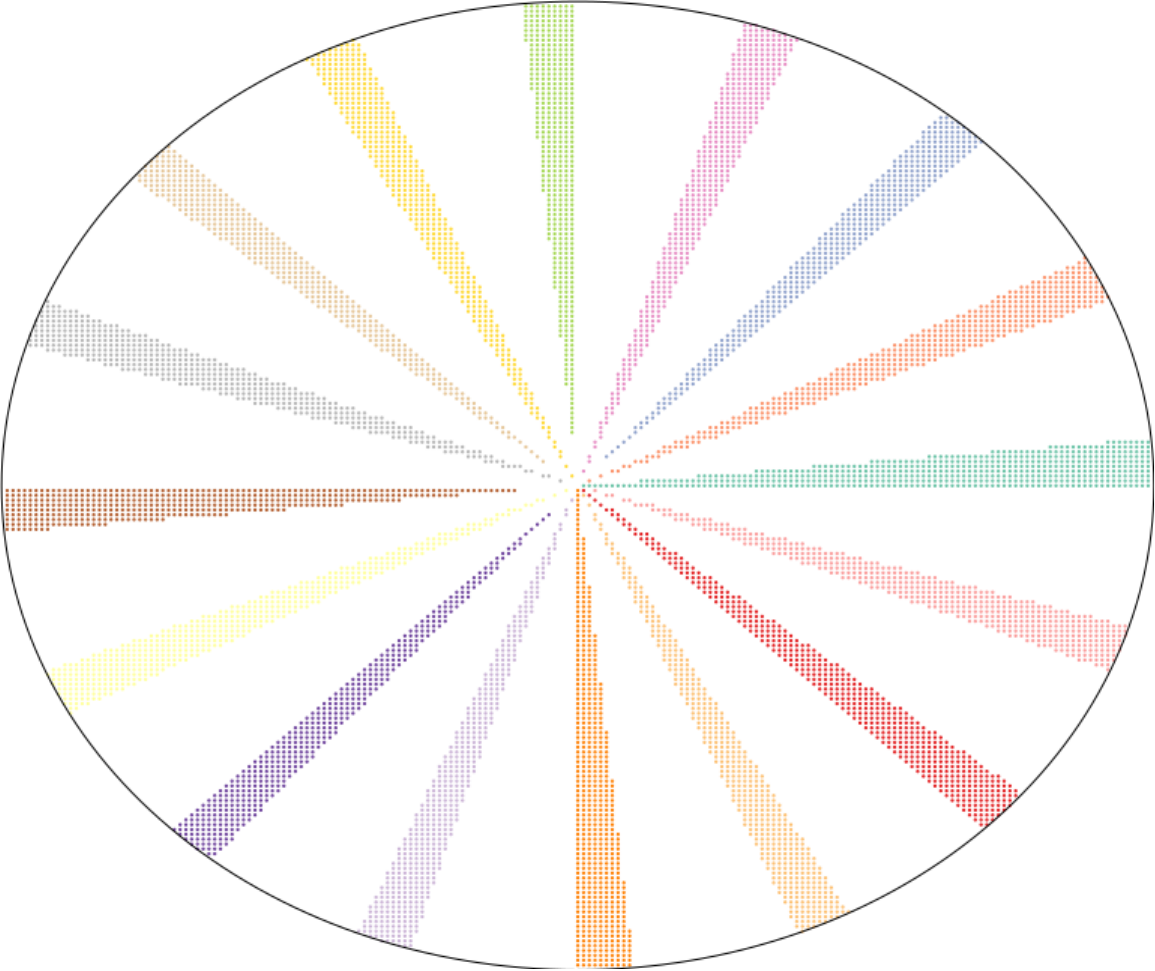} 
        \caption{Example}
        \label{fig: appendix|qualitative|nav2d|example}
    \end{subfigure}
    \begin{subfigure}[b]{1.0\textwidth}
        \centering
        \includegraphics[width=0.19\textwidth]{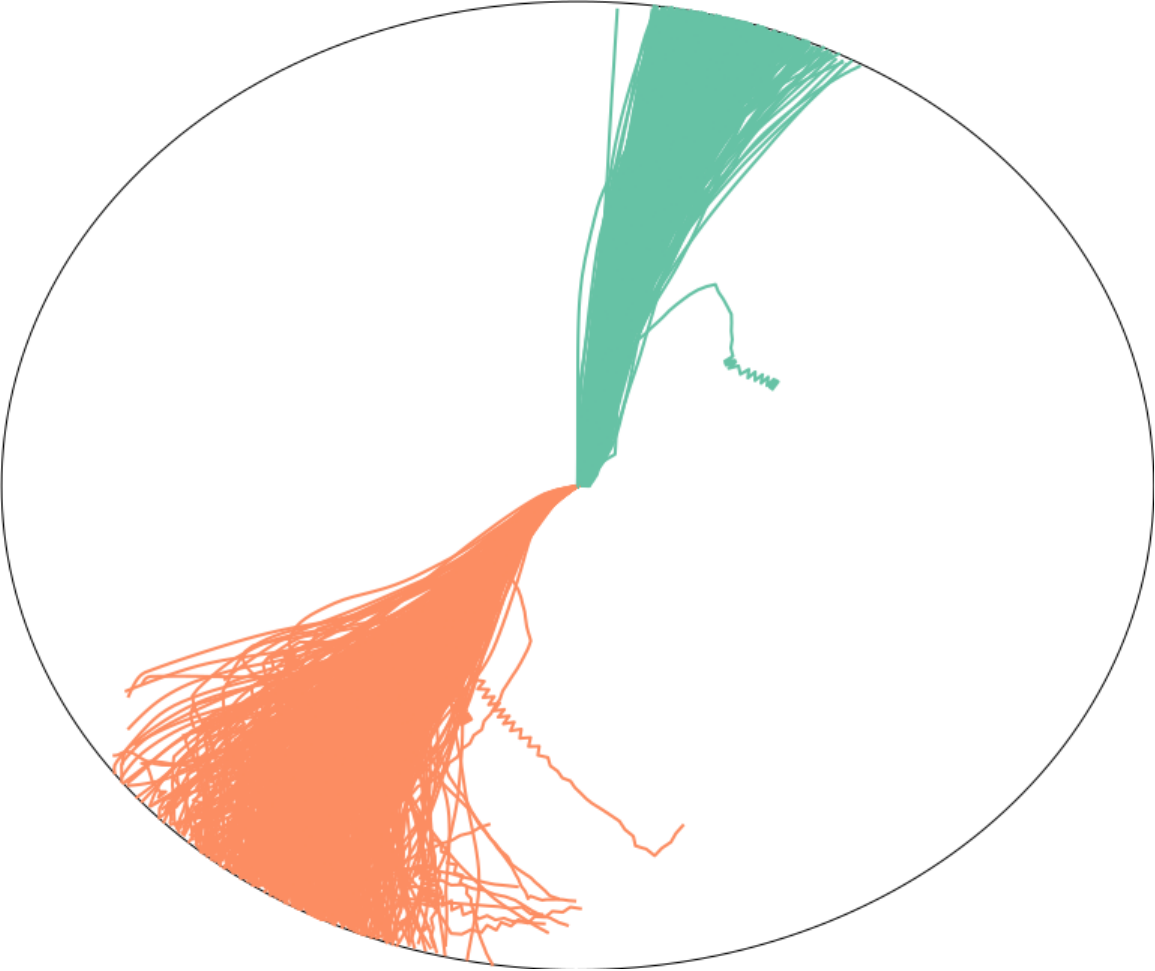} 
        \includegraphics[width=0.19\textwidth]{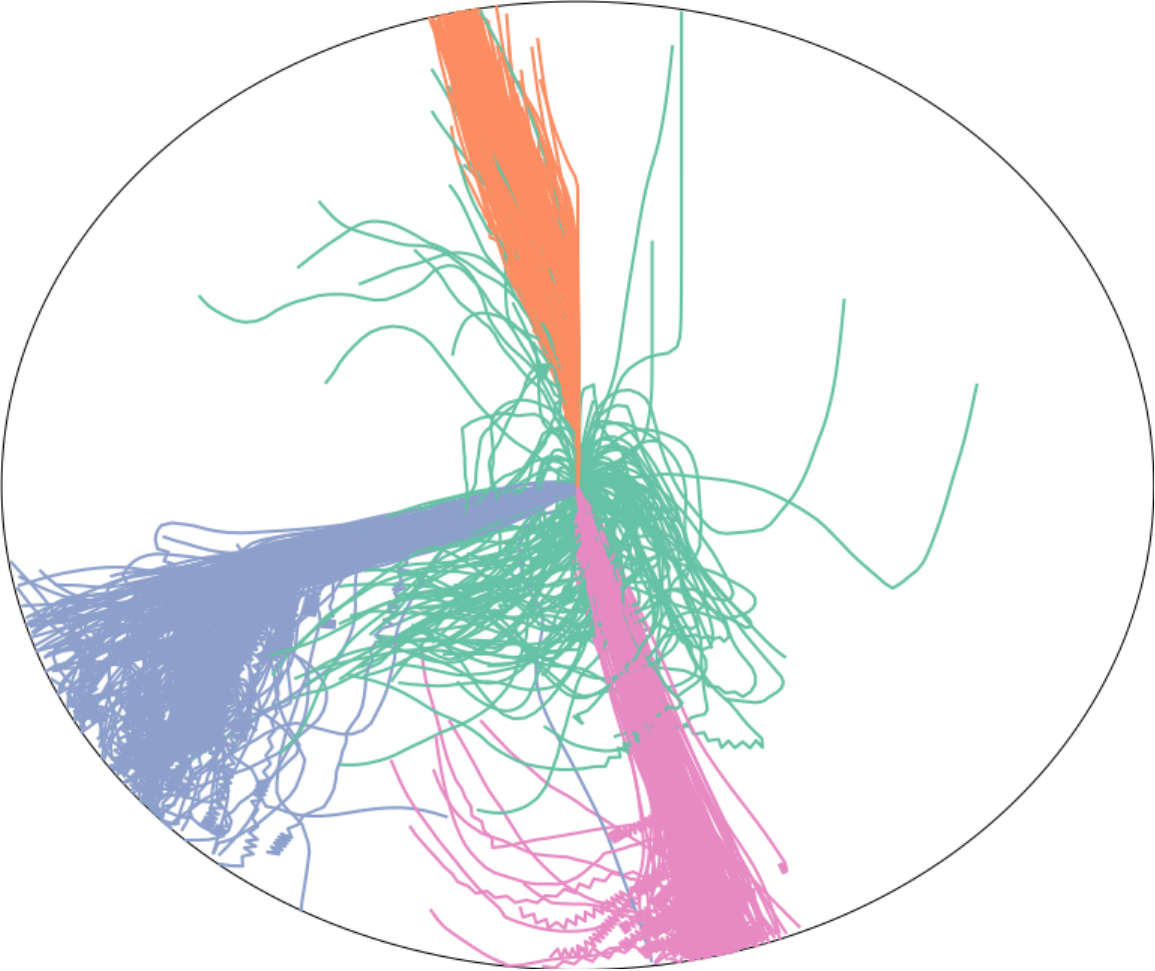} 
        \includegraphics[width=0.19\textwidth]{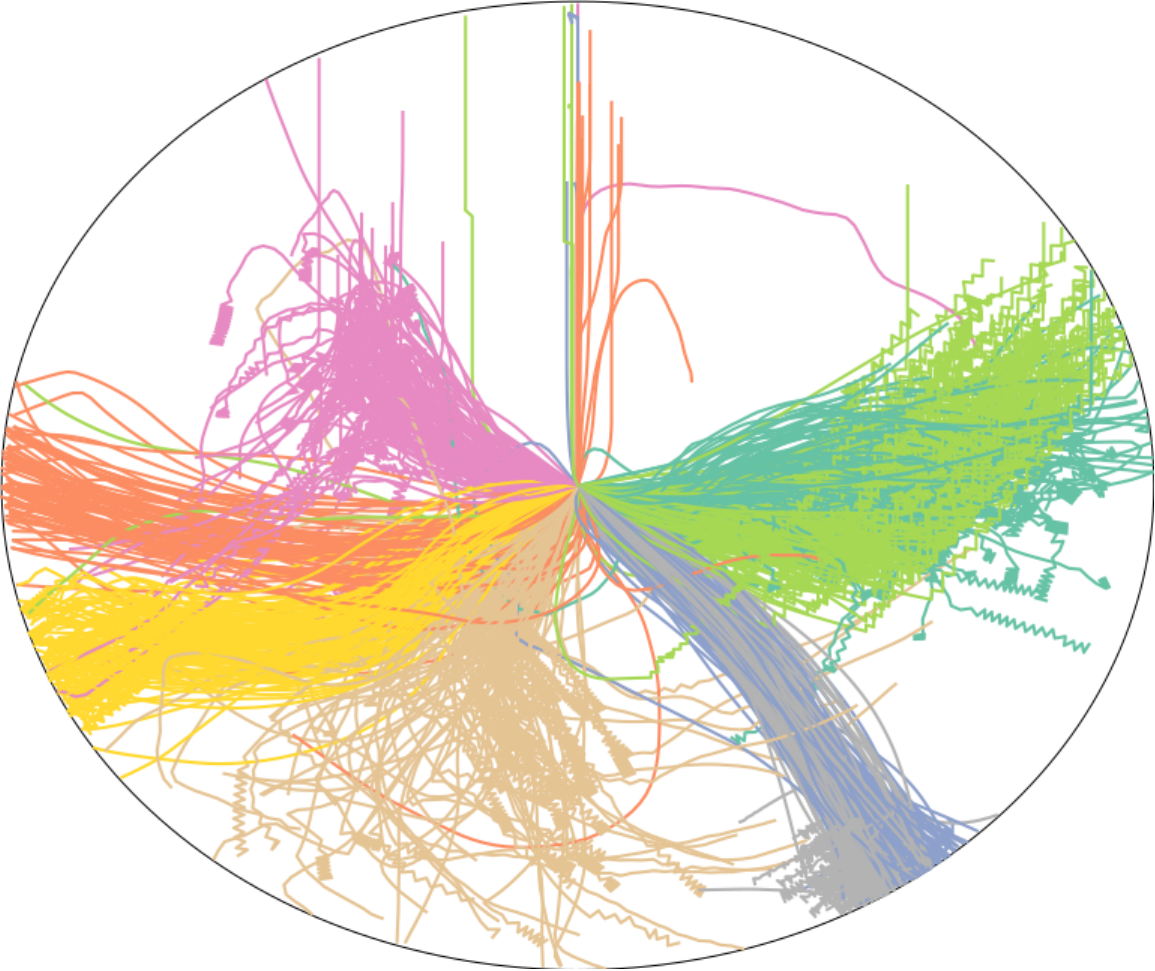}
        \includegraphics[width=0.19\textwidth]{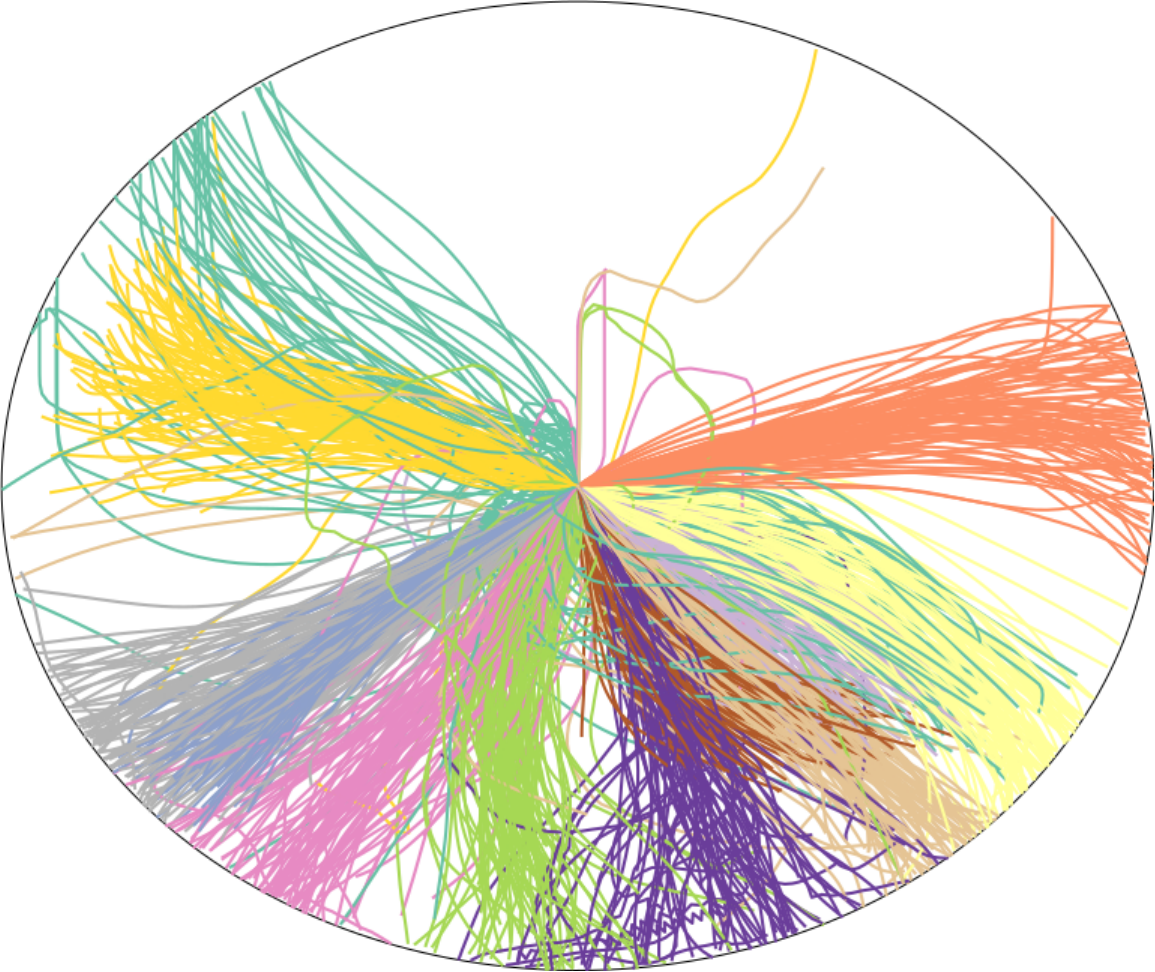} 
        \includegraphics[width=0.19\textwidth]{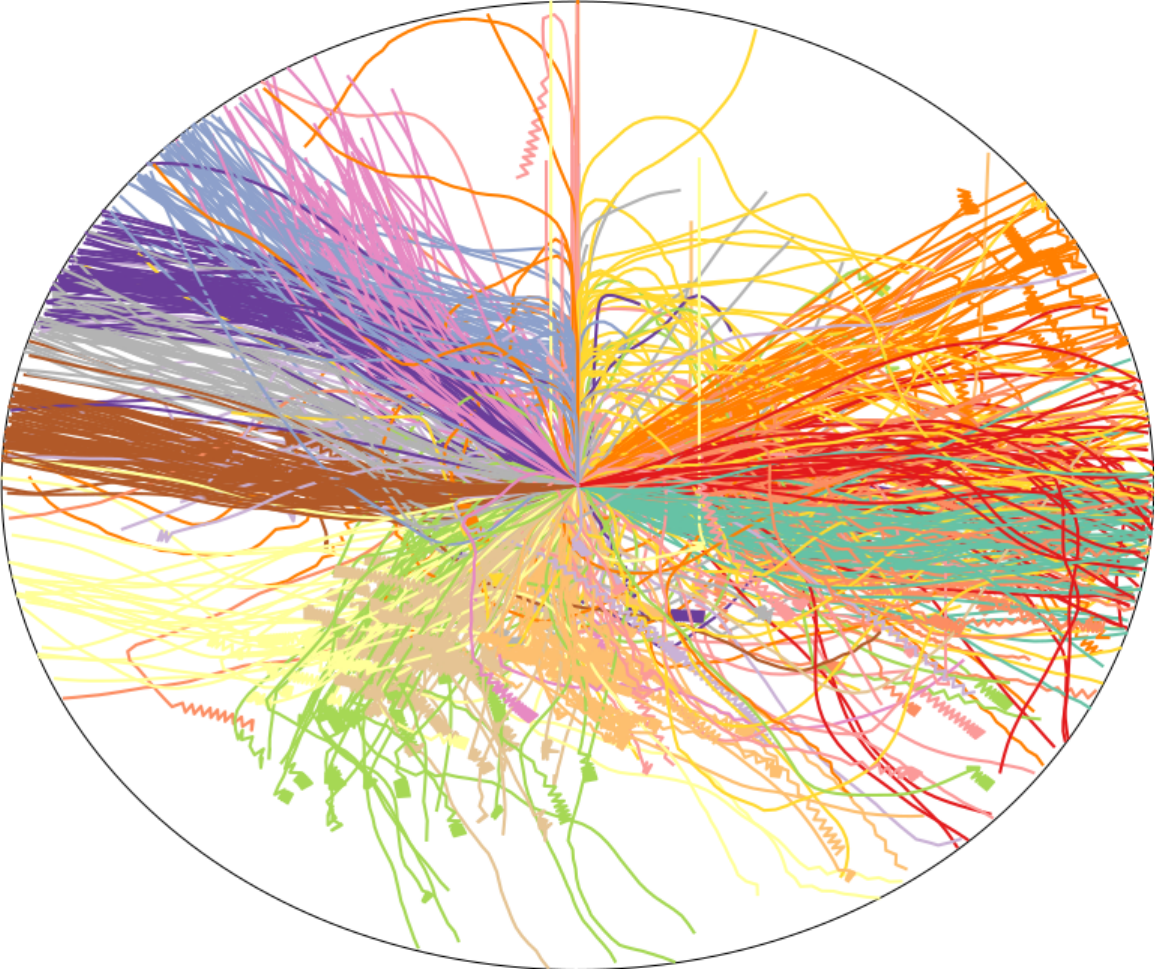}
        \caption{HaSD}
        \label{fig: appendix|qualitative|nav2d|hasd}
        
    \end{subfigure}
    \begin{subfigure}[b]{1.0\textwidth}
        \centering
        \includegraphics[width=0.19\textwidth]{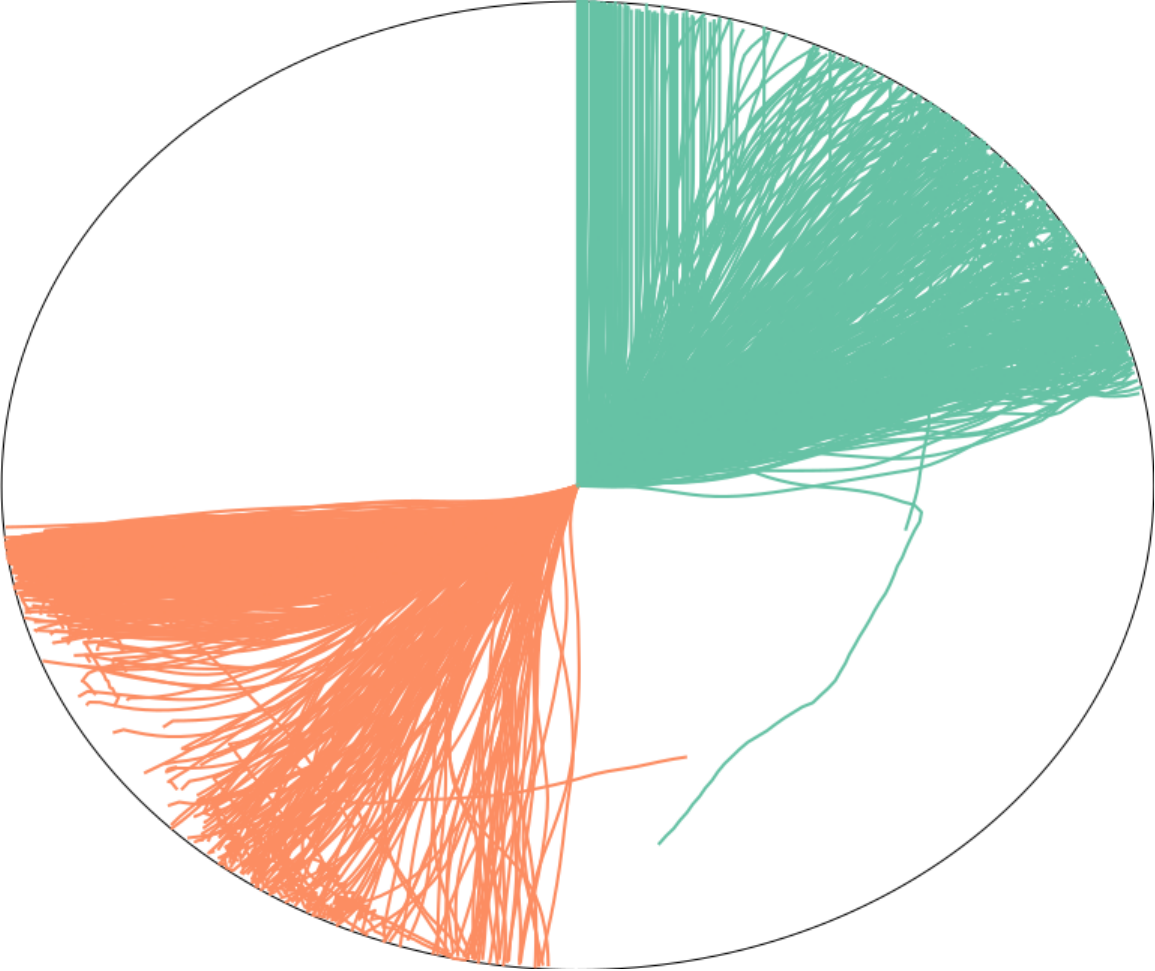} 
        \includegraphics[width=0.19\textwidth]{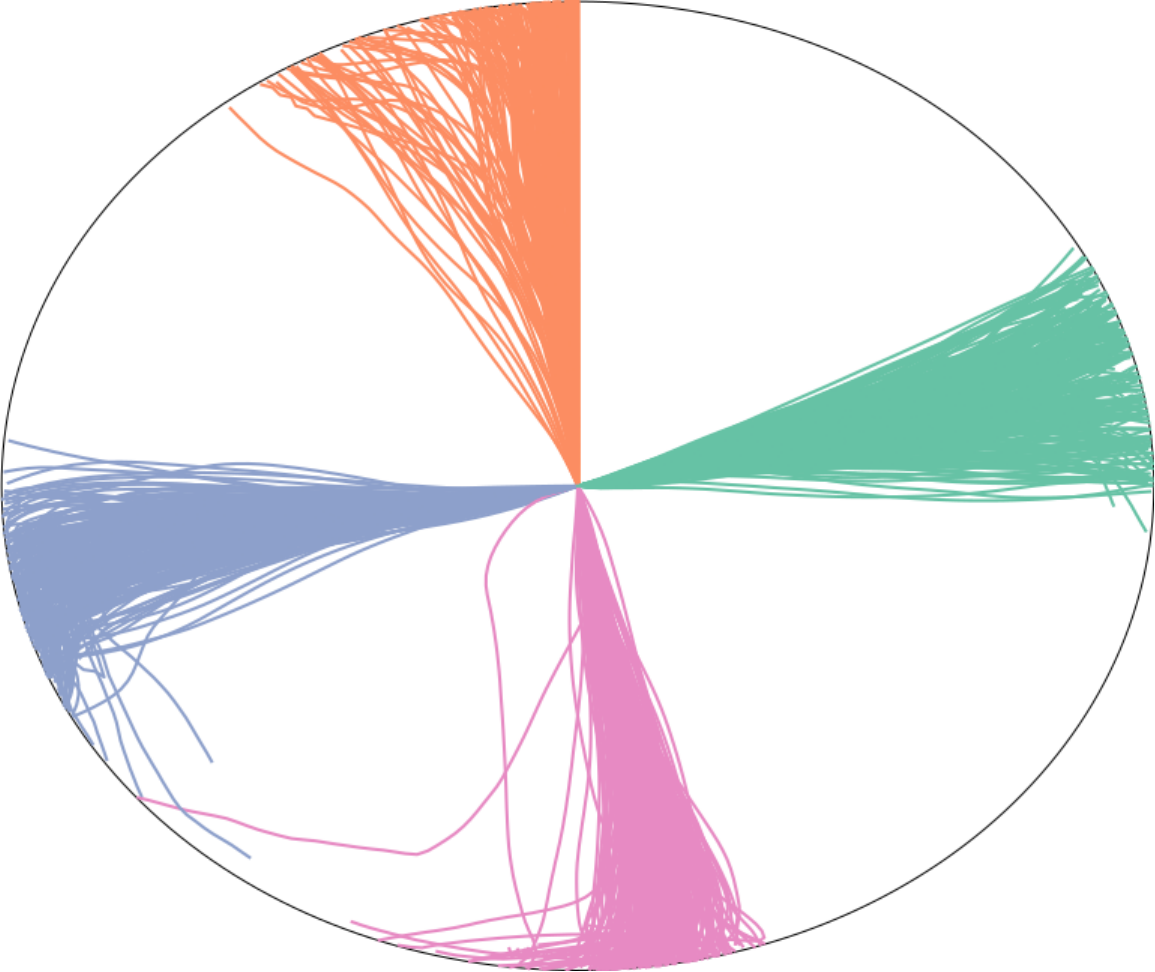} 
        \includegraphics[width=0.19\textwidth]{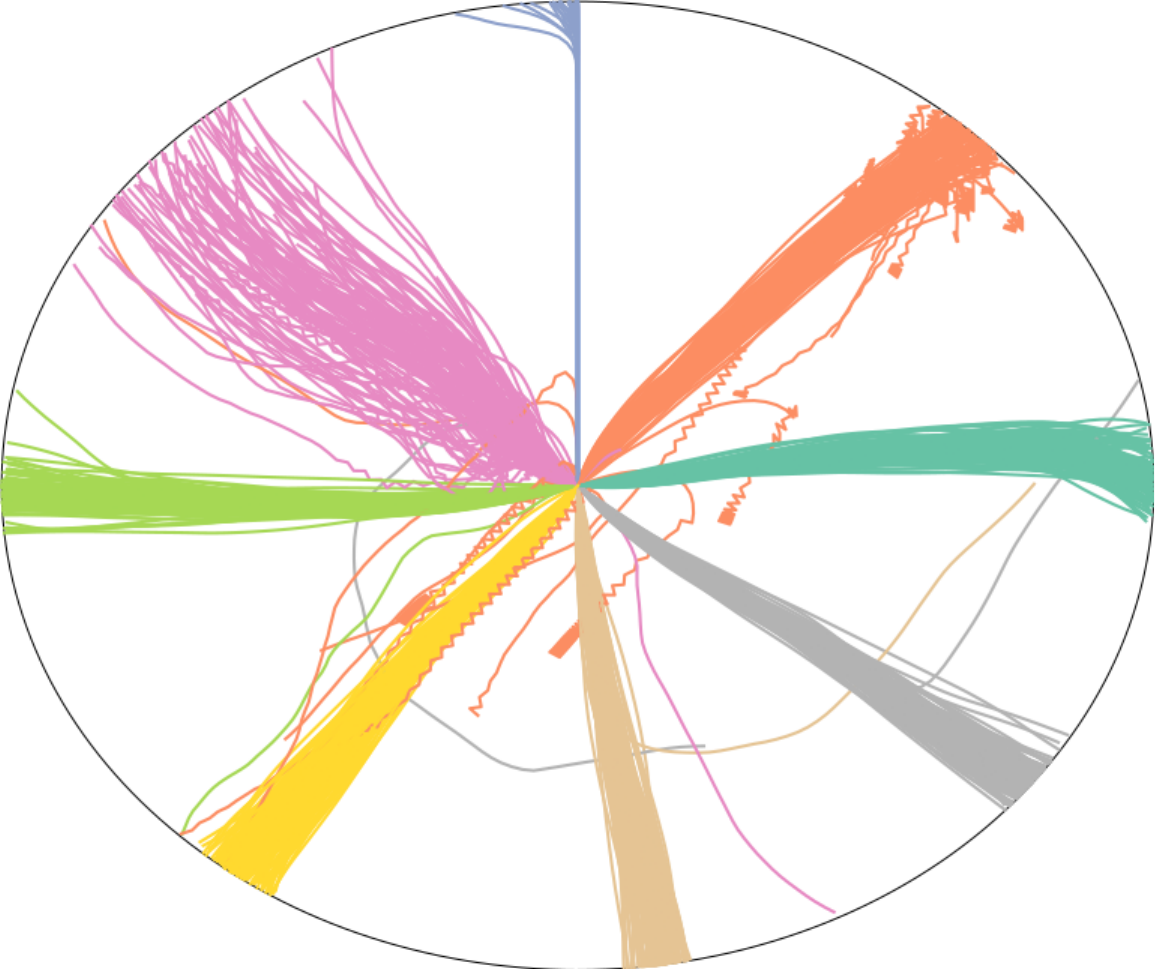}
        \includegraphics[width=0.19\textwidth]{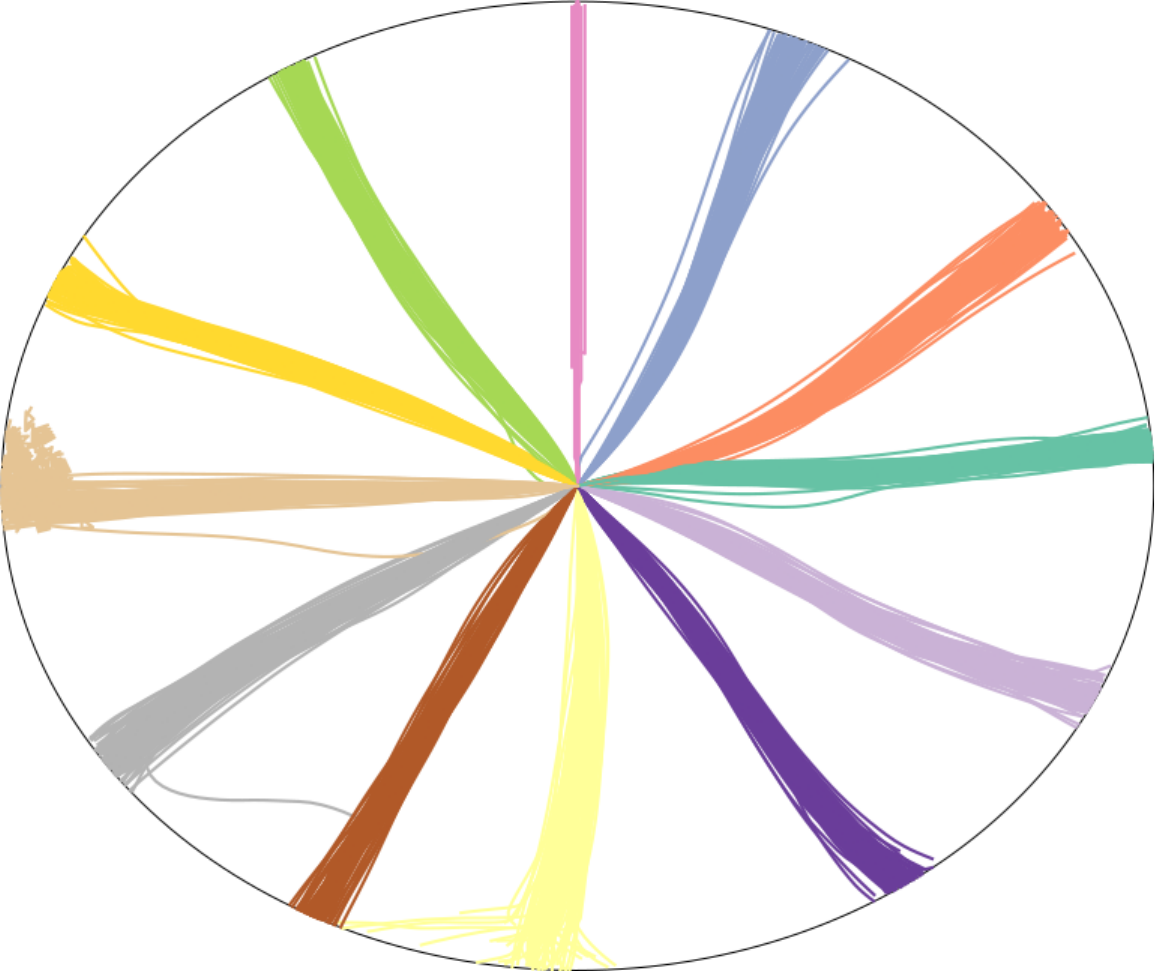} 
        \includegraphics[width=0.19\textwidth]{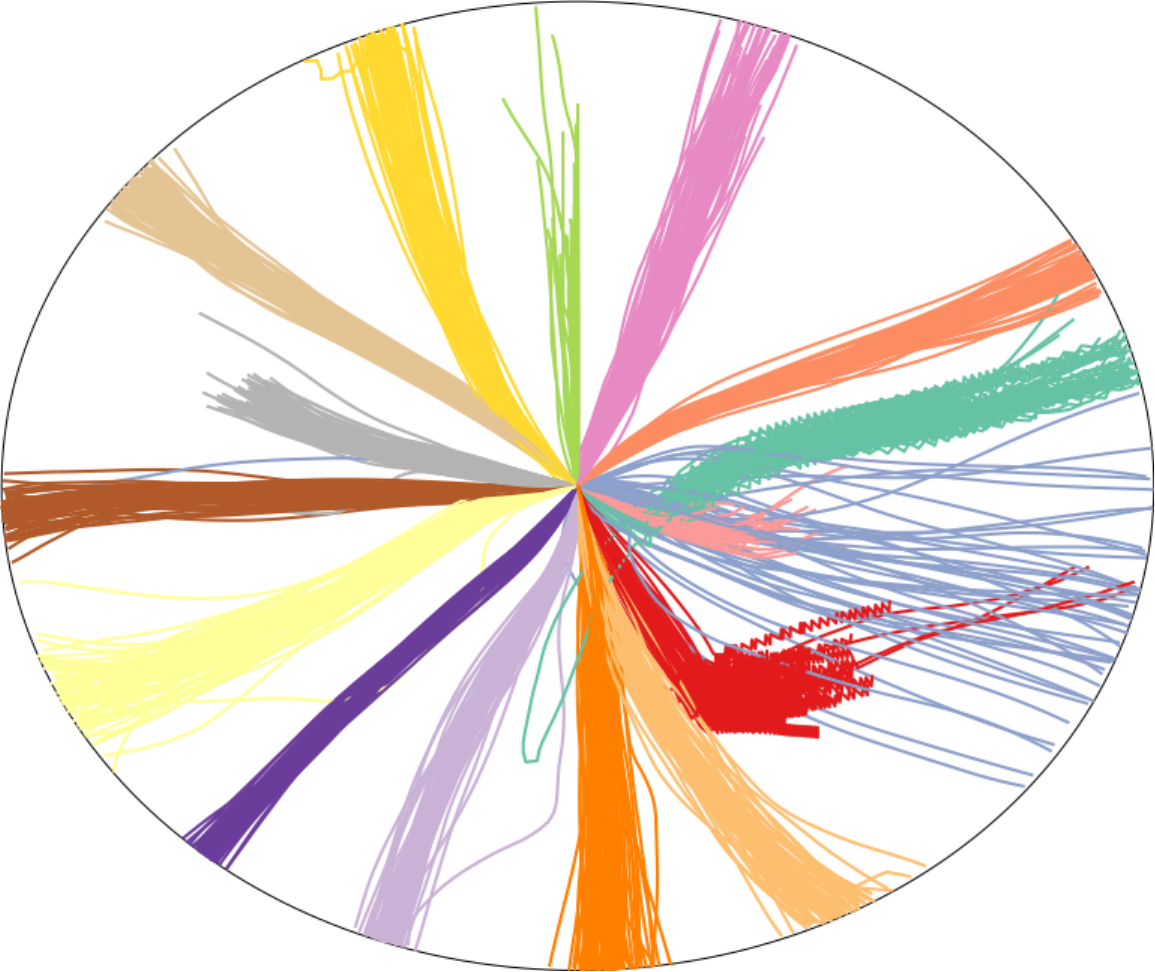} 
        \caption{SRSD}
        \label{fig: appendix|qualitative|nav2d|srsd}
    \end{subfigure}
    \begin{tikzpicture}
    \draw (-7.6,0) -- (2.15,0);
        \draw (-7.6,0.1) --++ (0,-0.2) node[below] {\strut 2};
        \draw (-5.1,0.1) --++ (0,-0.2) node[below] {\strut 4};
        \draw (-2.7,0.1) --++ (0,-0.2) node[below] {\strut 8};
        \draw (-0.27,0.1) --++ (0,-0.2) node[below] {\strut 12};
        \draw (2.15,0.1) --++ (0,-0.2) node[below] {\strut 16};
        \draw (-2.55,-0.5) node[anchor=north] {Number of Semantics};
     %\draw [thin,black,stealth-stealth] (-5,0) -- (5,0);
    \end{tikzpicture}
    \caption{(a) Examples of the desired sectors representing semantic behaviours for each 2D navigation environment with 2,4,8,12 and 16 semantic behaviours. 
    (b) Visualisation of skills learned with HaSD over each 2D navigation environments. (c) Visualisation of skills learned with SRSD over each 2D navigation environments}
    \label{fig: appendix|qualitative|nav2d}
\end{figure}

\subsubsection{(Q3) How sensitive is SRSD to feedback budget?}
\textbf{(Q3) Sensitivity to Human Feedback Budget.} We also analyse the sensitivity of SRSD performance to the human feedback budget. We evaluated SRSD with budgets of 40, 100, 400 and 1400 maximum feedback instances. We report in Figure~\ref{fig: downstream_task|zero_shot|general|feedbacks} the aggregated normalised score distribution against ComSD over zero-shot capabilities. While reducing feedback negatively impacts SRSD's zero-shot capability, it still performs better than ComSD even with as few as 40 feedback instances.  
\begin{figure}[htbp]
    \centering
    \includegraphics[width=0.75\textwidth]{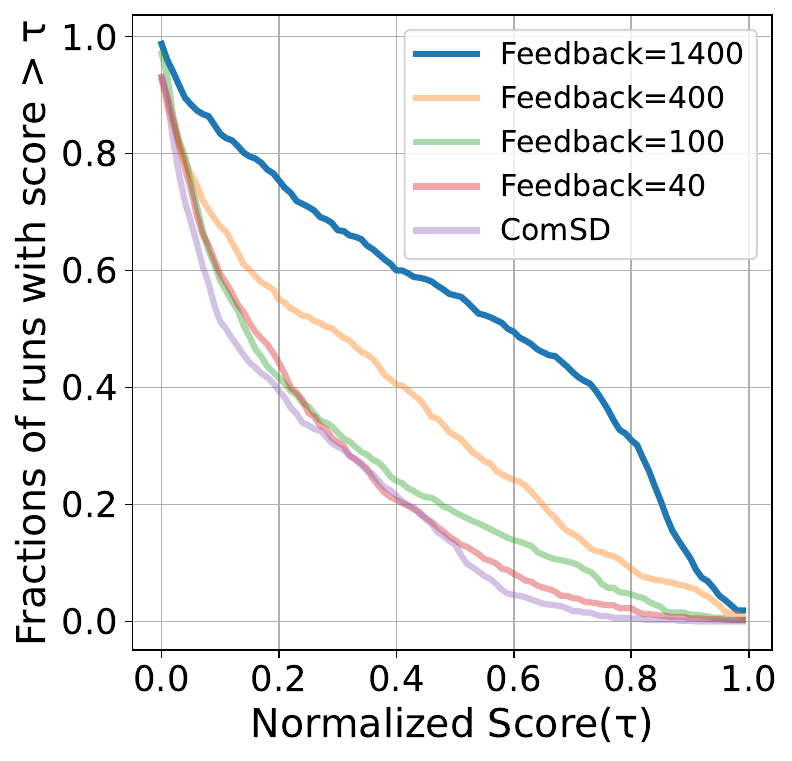}
    \caption{Aggregated and normalised score distributions for 12 downstream tasks on URLB over zero-shot evaluation. (a) shows the effect of different feedback budgets on SRSD performance. While reducing feedback impacts the zero-shot capability of SRSD, it still performs better than ComSD with a budget of 40.}
    \label{fig: downstream_task|zero_shot|general|feedbacks}
\end{figure}

\subsubsection{(Q4) Does active sampling ensure evenly distributed dataset of semantic labels?}
We analyse the effect of active sampling on label sample from the Human and of SRSD performances. Figure~\ref{fig: active_sampling|labels} shows the budget allocation percentage (\%) allocated to relevant semantic labels with and without active sampling. Without active sampling, more than 50\% of samples were allocated to irrelevant semantic, while, in the Quadruped and Hopper domains, only 15\% or 2\% of samples were labelled as one of the relevant semantics. In contrast, our active sampling scheme allocates more than 50\% of the samples to relevant classes in all domains and up to 79\% in Walker. This shows that it effectively reduces oversampling of irrelevant semantics. Additionally, in Figure~\ref{fig: active_sampling|jain} we measured the resource allocation fairness across relevant classes. To this end we used Jain's fairness index~\cite{jain1984quantitative} metric that measures the equitable distribution of resources among $|\mathcal{C}^+|$ relevant classes, where a value of 1 represents perfect fairness and a value of $1/|C^+|$ signifies the utmost unfairness (see practical implementation section in Appendix~\ref{subsec: apdx|metrics_details} for more details).  We can observe that our active sample also improves the fairness of samples allocated to each relevant class and allows us to collect labels in all relevant semantics. Lastly, we can observe in Figure~\ref{fig: downstream_task|zero_shot|general|ablation} that reducing oversampling of irrelevant semantic and improving fairness allocation in relevant classes can positively influence SRSD performance.

\begin{figure}[htbp]
    \centering
    \includegraphics[width=0.75\textwidth]{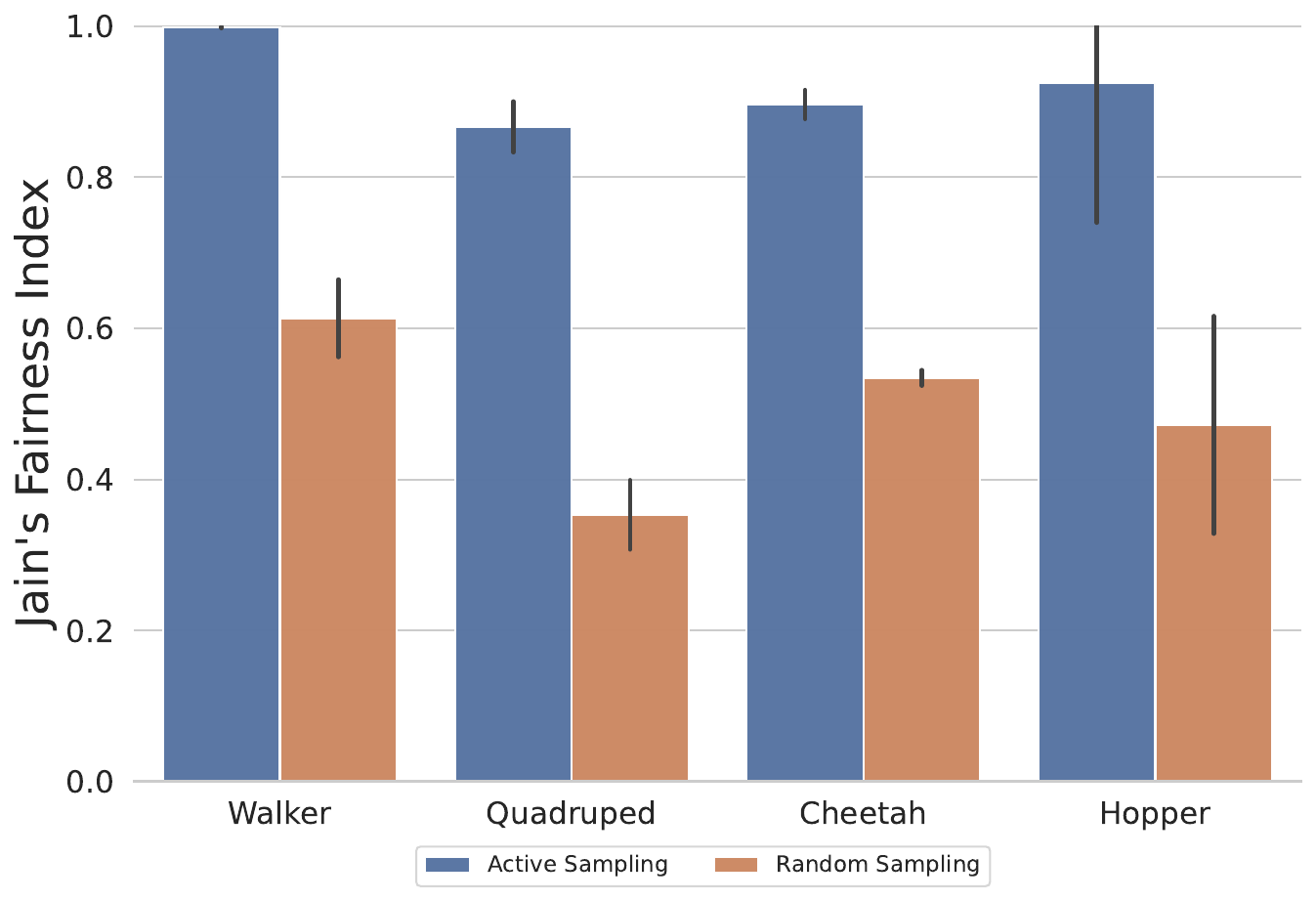} 
    \caption{Measure of allocation's fairness of labels across relevant classes (1 is a uniform distribution, while $1/|\mathcal{C}^+|$, (0.25 here), means that only one relevant class was labelled). Our active sampling methods ensure a balanced dataset unlike random sampling.}
    \label{fig: active_sampling|jain}
\end{figure}
\begin{figure}[htbp]
    \centering
    \includegraphics[width=0.75\textwidth]{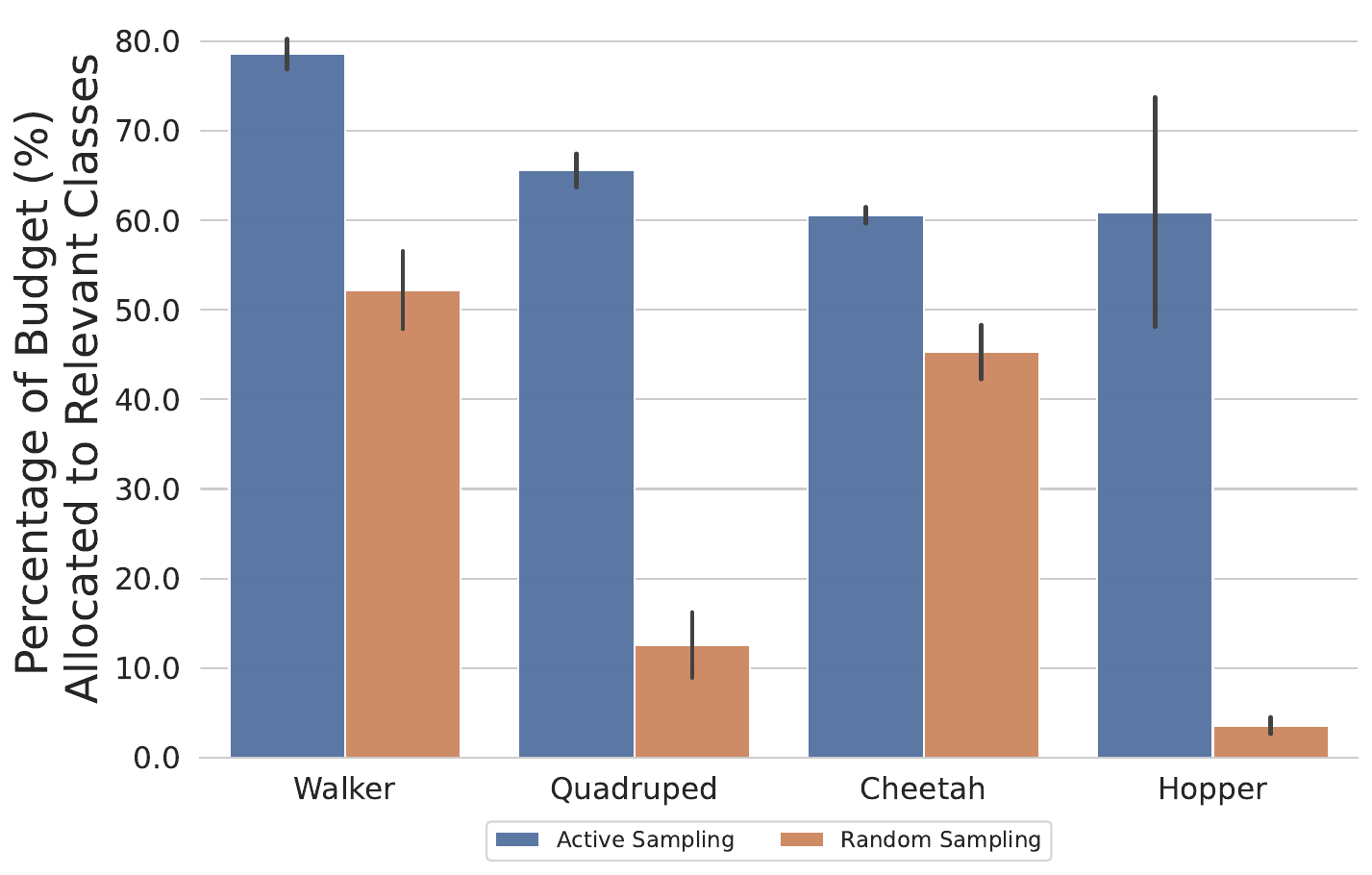} 
    \caption{Percentage of budget (\%) allocated to labels from the relevant semantics. Compared to random sampling, our active sampling allocated more labels to relevant semantics.}
    \label{fig: active_sampling|labels}
\end{figure}

\subsubsection{(Q5) Does incorporating a distributional perspective improve SRSD performance?}  
Figure~\ref{fig: downstream_task|zero_shot|general|ablation} indicates that integrating TQC into SRSD brings statistically meaningful performance gains, as per the Neyman-Pearson statistical testing criterion~\citep{bouthillier2021accounting}, (i.e., the probability of improvement exceeds 0.75 and the associated confidence intervals do not overlap 0.75). Further analysis reveals that the addition of explicit overestimation controls within TQC significantly improves outcomes, with statistical significance. Both components improve SRSD's performance, emphasising the importance of managing aleatoric uncertainty and mitigating overestimation, particularly in frameworks involving multiple intrinsic reward functions. 

\begin{figure}[htbp]
    \centering
    \includegraphics[width=0.75\textwidth]{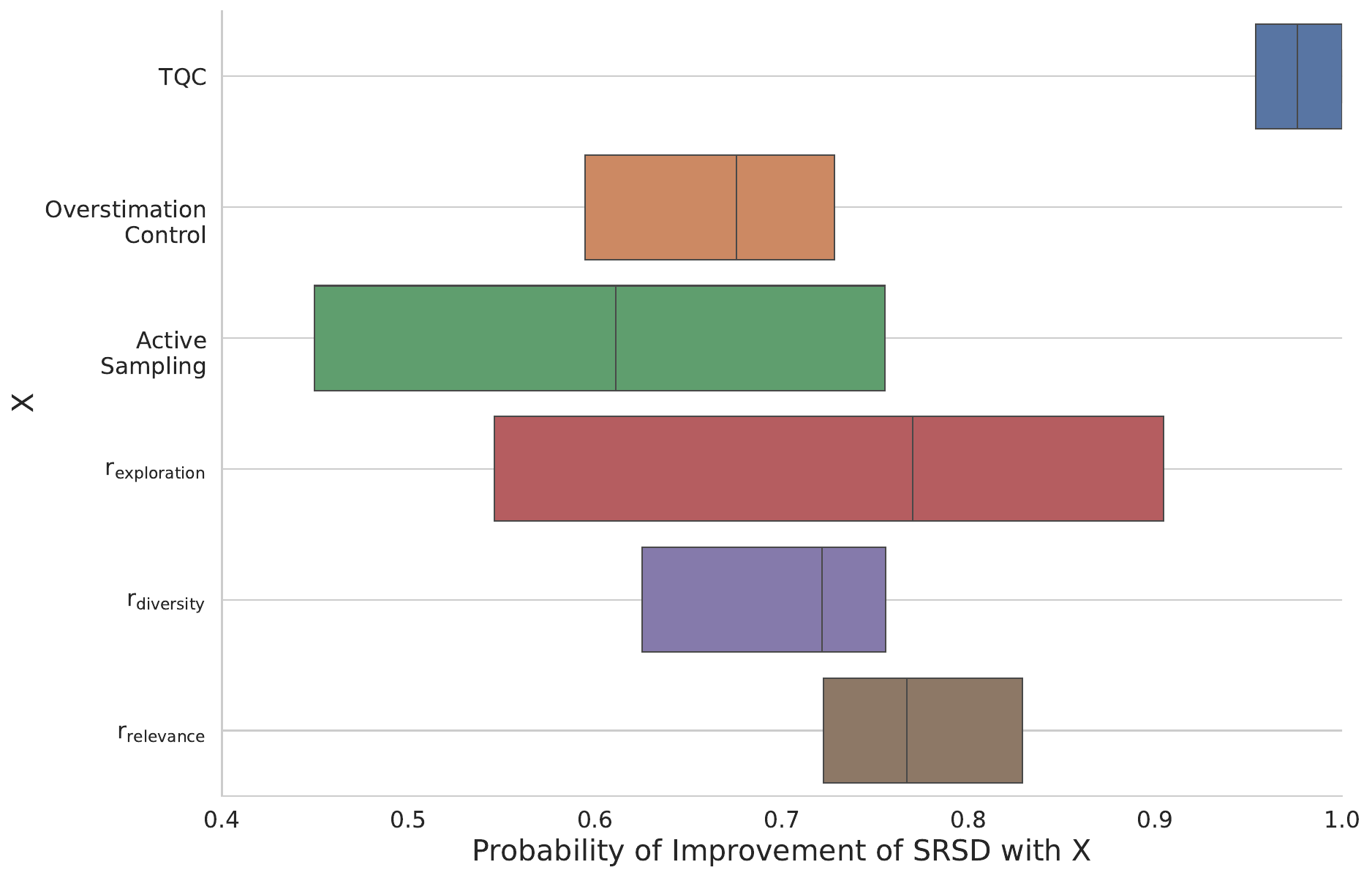}
    \caption{Ablation study of SRSD components. For each variant, the average probability of improvement (with confidence intervals) is reported, quantifying the impact of including these components in SRSD’s architecture.}
    \label{fig: downstream_task|zero_shot|general|ablation}
\end{figure}

\subsubsection{Q6) How does each reward term individually affect training performance?} 
Figure~\ref{fig: downstream_task|zero_shot|general|ablation} shows the average probability of improvement when each reward  signal is integrated into SRSD. All reward components achieved a probability above 0.5, indicating consistent performance gains. Both $r_\text{diversity}$ and $r_\text{relevance}$ exhibit statistically significant improvements, while $r_\text{exploration}$ shows greater variability. This variability likely arises because $r_\text{exploration}$ is most beneficial in tasks requiring exploration, leading to substantial gains in some cases, but more modest improvements when exploration is less critical.

\subsubsection{(Q7) How does SRSD perform with real Human-in-the-Loop Experiment?} 
We further evaluated our method with a human in the loop by replicating the 2D navigation experiment with four semantics. Authors familiar with the task provided 400 feedback signals across 12 query sessions (1.5 minutes each; $\approx$20 minutes total). To collect semantic labels for the human-in-the-loop experiment we implemented an interface modelled after~\citep{hussonnois2025human} and adapted it for semantic labelling (Figure\ref{fig: downstream_task|fine_tuning|tasks}). The layout is as follows: brief task instructions appear at the top (shared across all experiments). The left panel displays each trajectory segment as a single static image of the agent’s path, following~\citep{hussonnois2025human} to enable rapid feedback, for more complex behaviours (e.g., in the DeepMind Control suite), this can be replaced by a short video. The right panel provides labelling controls, including a predefined \textit{Irrelevant} class and an \textit{Add Semantic} option that lets annotators define new semantic categories up to a preset maximum. Selecting a label advances to the next segment in the selected batch, and annotators can navigate back to revise previous labels with \textit{Next} and \textit{Previous}. After all queries are labelled, the user saves them with \textit{Save}, generating a \textit{labels.json} file. Upon closing the interface, this file is automatically loaded and the SRSD algorithm resumes.

\begin{figure*}[htbp]
    \centering
    \includegraphics[width=0.95\textwidth]{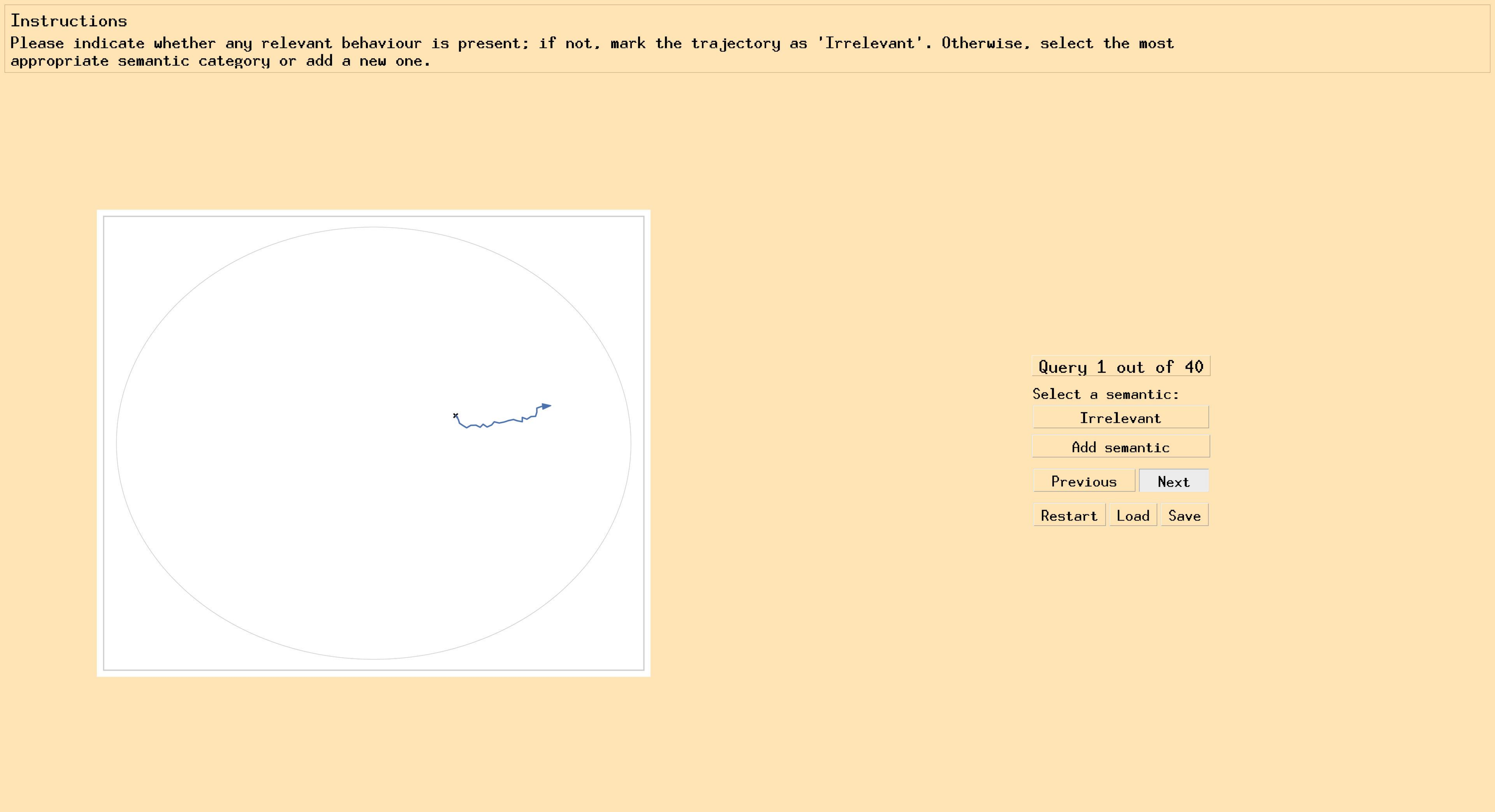} 
    \caption{Interface used to collect human semantic feedback during skill discovery.}
    \label{fig: appendix|feedback|2dnav|app}
\end{figure*}

Qualitatively, skill sets learned under simulated and real feedback are semantically and visually comparable, as shown in Figure~\ref{fig: appendix|feedback|2dnav|general} in the Appendix. We also measured a 30\% disagreement rate between simulated and human labels, likely due to the simulated and real annotators' stochasticity and, more importantly, mismatches between the simulator’s hand-designed reward and human judgements (e.g., trajectories that  are close but outside a target region are rejected by the simulator yet often deemed acceptable by humans). Despite label noise, our method remains robust.

\begin{figure}[htbp]
    \centering
    \begin{subfigure}[b]{0.45\textwidth}
        \centering
        \includegraphics[width=0.95\textwidth]{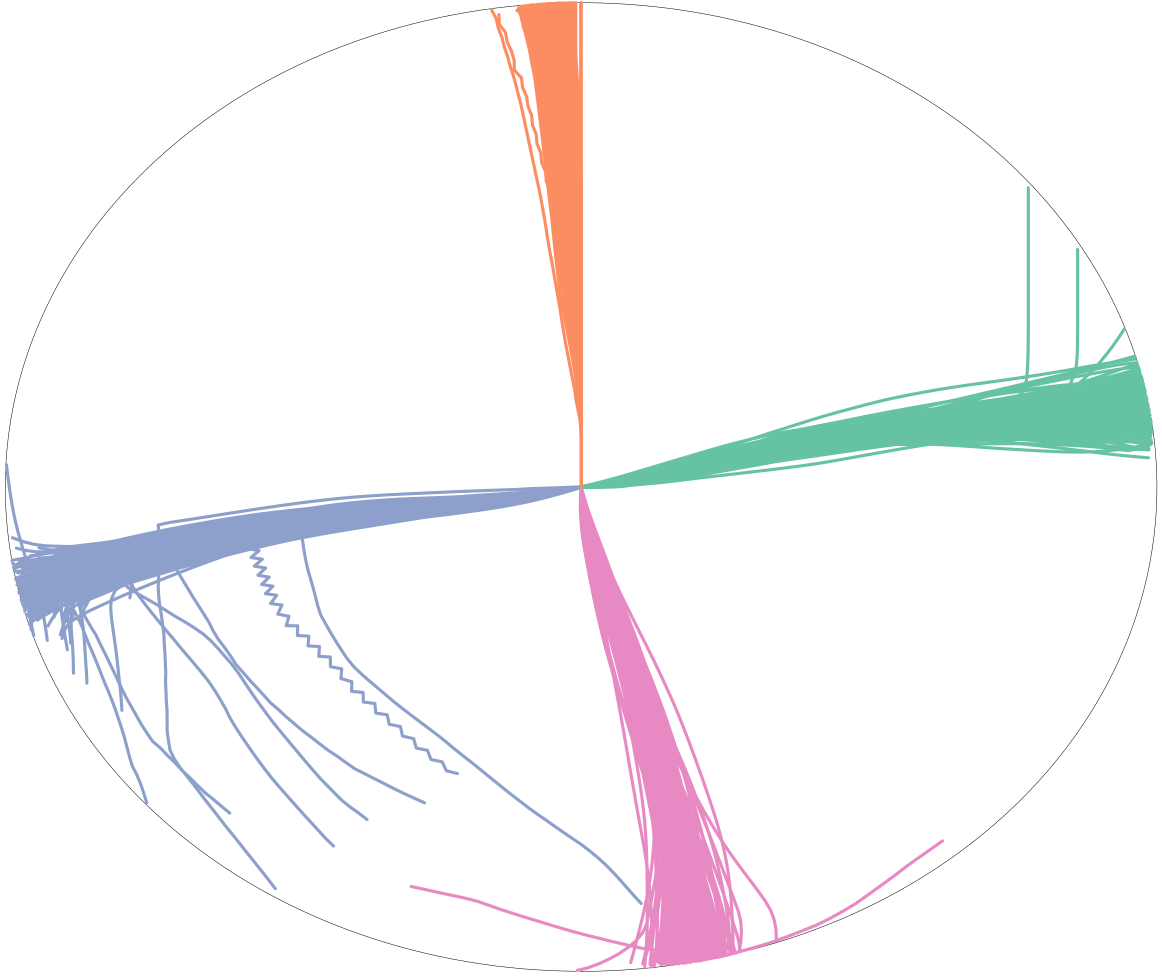} 
        \caption{Skill Set obtained from simulated human feedback.}
         \label{fig: appendix|feedback|2dnav|simulated}
    \end{subfigure}
    \hfill
    \begin{subfigure}[b]{0.45\textwidth}
        \centering
        \includegraphics[width=0.95\textwidth]{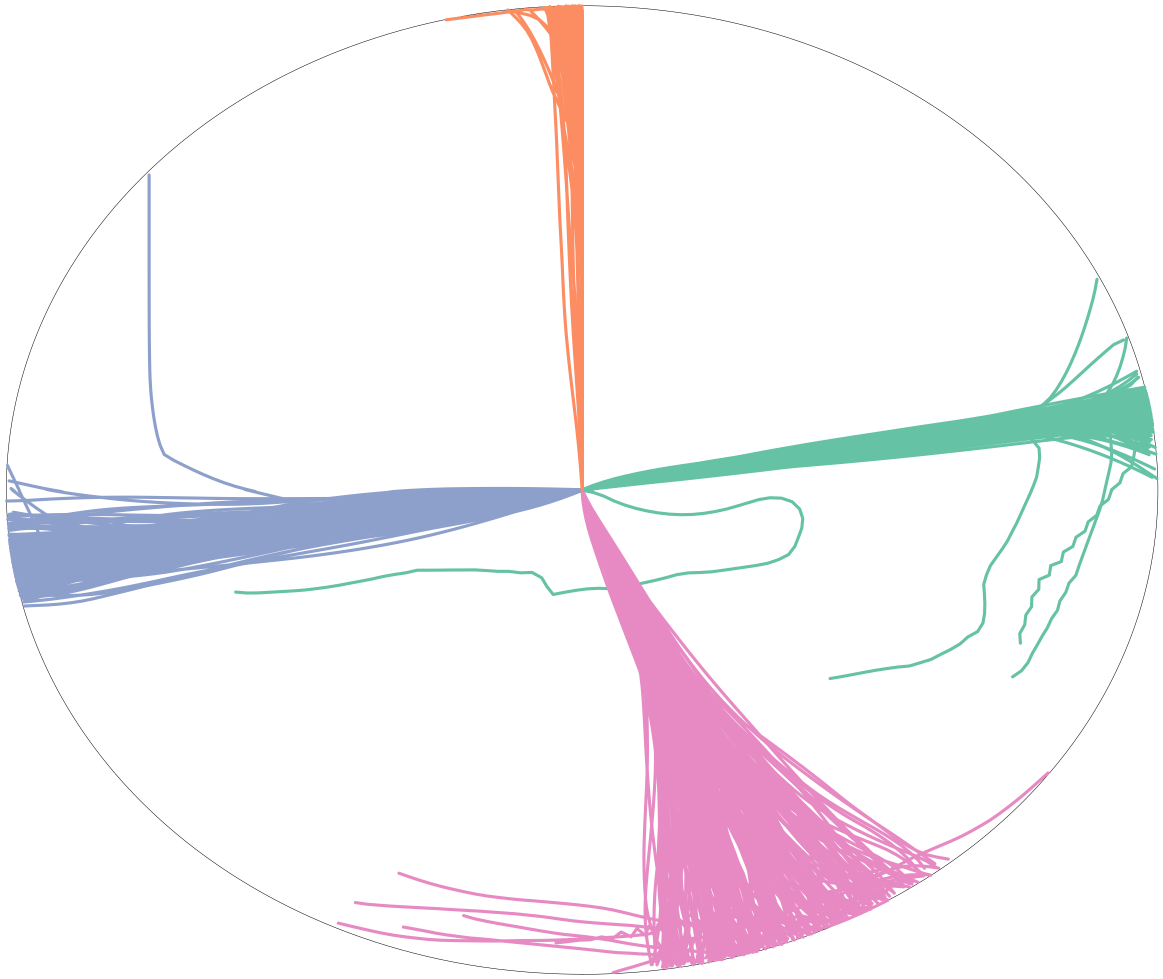} 
        \caption{Skill Set obtained from real human feedback.}
         \label{fig: appendix|feedback|2dnav|real}
    \end{subfigure}
    \caption{Visualisation of Skill Set obtained from simulated human feedback(a) and real human feedback (b). Both skill set are semantically and visually comparable.}
    \label{fig: appendix|feedback|2dnav|general}
\end{figure}

%% file: appendix/proof.tex
\newpage
\subsection{Scaling with respect to $|\mathcal{C}|$: Semantic vs Preference Feedback.}
\label{subsec: apdx|proof}
\textit{Proof.} We provide a more detailed analysis of Proposition \ref{proposition: scaling_c} by defining for each feedback mode the probability to sample a query that is informative about a relevant semantic. 
First, we assume semantics $\mathcal{C}$, being a set of $|\mathcal{C}|$ elements, to be composed of at least one irrelevant semantic (corresponding to all irrelevant types of behaviour) $\mathcal{C}^-$ and possibly multiple relevant semantic (relevant type of behaviour) $\mathcal{C}^+$.

\textbf{Case 1 $|\mathcal{C}|=1$} In this case, every behaviour belongs to the irrelevant semantic, thus both modes of feedback are inefficient.

\textbf{Case 2 $|\mathcal{C}|=2$} In this case, there is only one relevant semantic along with the irrelevant semantic. This setting correspond to the initial setting of preference-based reinforcement learning, where semantic-label feedback is not necessary.

\textbf{General Case $|\mathcal{C}|\geq3$:} In this case, we assume $\sigma$ a segment sampled and $p$ is the probability of $\sigma$ to belongs to the irrelevant semantic. Additionally, we assume that relevant semantics are evenly distributed among $C^+$, thus $p(\sigma \in c_i|c_i \in \mathcal{C}^+) = \frac{1}{|\mathcal{C}|-1}$.

For semantic-label feedback, a query that is informative about a relevant semantic corresponds to any query with a segment that is not irrelevant semantically. Thus:
\begin{align*}
p_{\text{semantic}}&=p(\sigma \notin \mathcal{C}^-) \\
&=1-p  \\ 
&\rightarrow O(1) \qquad \text{ as it is independent of $|\mathcal{C}|$}
\end{align*}

For preference feedback, a query is informative about a selected relevant semantic $c_i$ if the query is composed of a segment $\sigma_1$ of the selected relevant semantic and a segment $\sigma_2$ from any other semantics. Thus :
\begin{align*}
p_{\text{preference}}&=p(\sigma_1\in c_i, \sigma_2 \in c_{j,j\neq i })
\end{align*}
Each segment is sampled independently, so we get:
\begin{align*}
p_{\text{preference}}&=p(\sigma_1\in c_i) \times p(\sigma_2 \in c_{j,j\neq i })
\end{align*}
The probability to select a segment from a selected relevant semantic $p(\sigma_1\in c_i)$ is the probably of selecting a segment that is not irrelevant semantically times the probability selecting a relevant semantic among all relevant semantics which correspond to :
\begin{align*}
p(\sigma_1\in c_i)= p(\sigma_1\notin \mathcal{C}^-) \times p(\sigma \in c_i|c_i \in \mathcal{C}^+) = (1-p) \times (\frac{1}{|\mathcal{C}|-1}) 
\end{align*}
Thus we finally get:
\begin{align*}
\begin{split}\label{eq: apdx|p_pref}
p_{\text{preference}}&= (\frac{1-p}{|\mathcal{C}|-1}) \times (1 -( \frac{1-p}{|\mathcal{C}|-1}))\\ 
&=( \frac{1-p}{|\mathcal{C}|-1}) \times( \frac{|\mathcal{C}|-2+p}{|\mathcal{C}|-1} )\\
&\rightarrow  O(\frac{1}{|\mathcal{C}|}) \\
\end{split}
\end{align*}

%% file: appendix/implementation_details.tex
\subsection{Practical Implementation}
\label{subsec: apdx|implementation_details}
\subsubsection{Defining $z$ and practical implementation of $f$}
In our experiment, the skill $z \in \mathbb R^N$ is a continuous vector of dimension $N$. Of the $N$ dimension, the $N-|\mathcal{C}^+|$first elements are sampled independently of a standard normal distribution and then optionally normalised. The remaining $|\mathcal{C}^+|$ elements represent a one-hot vector and specify the desired relevant semantic class. The mapping function $f: Z '\rightarrow \mathcal{C}^+$ is typically implemented by reading this one-hot segment of $z$ to determine the relevant class associated with a particular skill vector, though any suitable mapping could be used in principle. 

\subsubsection{Additional Practical Implementation}
In this section, we detail practical implementation mechanisms that were not found to significantly improve the performance of SRSD.

\textit{Pseudo-labelling:} For SRSD and HaSD we followed previous work~\citep{park2022surf} work by incorporating pseudo-labels and data-augmentation during the training of the semantic predictor(SRSD) and the reward models(HaSD). Although we observed limited improvements over our method, we believe that leveraging the large dataset of diverse yet unlabelled behaviours that skill discovery inherently generates to improve sampling strategy and semantic-predictor capability is a promising research direction.

\textit{Skill Sampling:} In SRSD, each skill is associated with a latent variable $z$, whose final $|\mathcal{C}|-1$ components are replaced with a one-hot vector representing the selected relevant semantic skills performed. To select which relevant semantics to execute, a naive approach would be to uniformly sample all available semantics. However, this can result in imbalanced exploration, where certain semantics are easier to discover and thus receive more feedback.To address this, we adopt an adaptive sampling strategy: we track the number of feedback instances collected for each relevant semantic and set the sampling probability inversely proportional to this count. This encourages the selection of semantics with less feedback. This promotes more balanced coverage across all semantic classes and enables more uniform exploration and annotation of skills. This mechanism did not yield much impact, partially due to the active sampling scheme that balances feedback collected among semantics.

\subsubsection{Implementation Details and Hyperparameters}
\paragraph{Environments:} We used a (\href{https://github.com/shufflebyte/gym-nav2d}{custom 2D Navigation}) publicly released repository to experiment in simpler navigation environments that we customise for our experiment. For Deepmind Control Suite we use the the DMC custom task released by\href{https://github.com/rll-research/url_benchmark}{URLB}.
\paragraph{Algorithms:} As our backbone reinforcement learning algorithm, we use the TD3 implementation from the publicly available codebase LeanRL\cite{huang2022cleanrl}. For ComSD we use the publicly available code of \href{https://github.com/rll-research/cic}{CIC}. For TQC we used the publicly available codebase for (\href{https://github.com/SamsungLabs/tqc_pytorch?tab=readme-ov-file}{TQC}). For HaSD we adapted the publicly available codebase for (\href{https://github.com/HussonnoisMaxence/HaSD-AAMAS}{HaSD}) to fit a similar structure than SRSD. We report other hyper parameters used in Tables \ref{td3-table},  \ref{tqc-table} for TQC, \ref{comsd-table}, \ref{rlhf-table} for skill discovery and \ref{sf-table} for skill fine-tuning.
\paragraph{Resources.} All experiments are conducted on an Ubuntu 20.4 server with 36 cores CPU, 767GB RAM, and a V100 32GB GPU with CUDA version 12.0, each run in Sections \ref{sec: experiments} took  around 10h for pre-training and 15 minutes for skill fine-tuning.

\begin{table}[htbp]
  \caption{Hyperparameter TD3 Skill Discovery}
  \label{td3-table}
  \centering
  \begin{tabular}{ll}              \\
    \toprule
    Hyperparameter     & Value    \\
    \midrule
    Training iteration          & 1M(Nav2D), 2M(DNC)       \\
    Learning rate critic       & \begin{math}3.0 \times 10^{-04}\end{math} ,\begin{math}1.0 \times 10^{-04}\end{math}       \\
    Learning rate actor         & \begin{math}3.0 \times 10^{-04}\end{math} , \begin{math}1.0 \times 10^{-4}\end{math}     \\
    Update policy frequency     & 2       \\
    Update-to-data              & 1       \\
    Optimiser                   &  Adam \\
    Minibatch size              &  1024         \\
    Discount factor $\gamma$    &  0.99         \\
    Replay buffer size          &  \begin{math}10^6\end{math}         \\
    Hidden layers               &  2        \\
    Hidden units per layers     & 1024(Nav2D), 512(DMC)        \\
    Target network smoothing coefficient      &  0.995        \\
    Exploration noise          &  0.2        \\
    Policy noise: & 0.2\\
    Noise clip: & 0.5\\
    \bottomrule
  \end{tabular}
\end{table}

\begin{table}[htbp]
  \caption{Hyperparameter ComSD}
  \label{comsd-table}
  \centering
  \begin{tabular}{ll}              \\
    \toprule
    Hyperparameter     & Value    \\
    \midrule
    Learning rate discriminator       & \begin{math}1.0 \times 10^{-04}\end{math} (Nav2d)  \\
    Hidden layers      &  2        \\
    Hidden units per layers      & 1024   \\
    Z dim  & 64(Nav2D), 32(DMC) \\
    Z vector & $z\sim \mathcal N(0,1)$ continuous \\
    Z update frequency & 50 \\
    k KNN &  16 \\
    T      & 0.5    \\    
    \bottomrule
  \end{tabular}
\end{table}

\begin{table}[htbp]
  \caption{Hyperparameter Skill Labelling learning}
  \label{rlhf-table}
  \centering
  \begin{tabular}{ll}              \\
    \toprule
    Hyperparameter     & Value    \\
    \midrule
    Learning rate       & \begin{math}3.0\times 10^{-4}\end{math}      \\
    Optimiser           &  Adam \\
    Minibatch size      &     \\
    Size segment     &    25      \\
    Ensemble size     &  5         \\
    Hidden layers  & 2 \\
    Hidden units per layers & 256 \\
    Queries per feedback session       &  140    \\
    Frequency of feedback session       &  15K(Nav2D), 20k(DMC)     \\
    Start feedback      &  10k(Nav2D), 20k(DMC  \\
    |C|  & 3,5,9,13,17 (Nav2d), 5(DMC)\\
    \bottomrule
  \end{tabular}
\end{table}

\begin{table}[htbp]
  \caption{Hyperparameter TD3 Skill Finetunning}
  \label{sf-table}
  \centering
  \begin{tabular}{ll}              \\
    \toprule
    Hyperparameter     & Value    \\
    \midrule
    Training iteration          & 100K(       \\
    Learning rate critic       & \begin{math}1.0 \times 10^{-04}\end{math}       \\
    Learning rate actor         & \begin{math}1.0 \times 10^{-04}\end{math}     \\
    Update policy frequency     & 2       \\
    Update-to-data              & 1       \\
    Optimiser                   &  Adam \\
    Minibatch size              &  512         \\
    Discount factor $\gamma$    &  0.99         \\
    Replay buffer size          &  \begin{math}10^6\end{math}         \\
    Hidden layers               &  2        \\
    Hidden units per layers     & 1024        \\
    Target network smoothing coefficient      &  0.995        \\
    Exploration noise          &  0.2        \\
    Policy noise: & 0.2\\
    Noise clip: & 0.5 \\
    \bottomrule
  \end{tabular}
\end{table}

\begin{table}[htbp]
  \caption{Hyperparameter TQC}
  \label{tqc-table}
  \centering
  \begin{tabular}{ll}              \\
    \toprule
    Hyperparameter     & Value    \\
    \midrule
    Number of quantiles: & 25\\
    Number of networks: & 3\\
    Number of top quantiles to drop: 2 (pre-training), 5(fine-tuning) \\
    \bottomrule
  \end{tabular}
\end{table}

%% file: appendix/metrics.tex
\subsection{Practical Implementation}
\label{subsec: apdx|metrics_details}
\subsubsection{Jain's Fairness Index}

Jain’s fairness index~\cite{jain1984quantitative} is a metric commonly used for evaluating fairness in static resource allocation among $n$ users, such as dividing a cake into $n$ slices. If each users $i$ receives $x_i \in \mathbb{R}^+$, Jain’s index $Fairness(x)$ produces a value between $0$ and $1$, and is defined as: 

\begin{equation} 
Fairness(x)= \frac{\left(\sum{i=1}^n x_i\right)^2}{n \sum_{i=1}^n x_i^2} \label{eq:jain}.
\end{equation}
The higher the value of $Fairness(x)$, the fairer the resource-allocation process with the maximum value of 1 indicating that all users receive the same share. Conversely, the lower the value, the more unfair the allocation is, with the minimum attainable value being $Fairness(x) = 1/n$. This occurs when a single users receives all resources and the others receive none.

%% file: figures/algo.tex
\begin{algorithm}[htbp]\small
\textbf{Initialise} \begin{math}\mathcal B, \zeta^b\end{math}, \begin{math}\pi_\theta, q_\psi, g_{\phi_{1,2}}  \end{math} \;
\For{each epoch}{ 
    \For{each episode }{
        Sample trajectory \begin{math}\tau\end{math} with \begin{math}\pi_\theta(a_t|s_t,z)\end{math} and \begin{math}z \sim p(z)\end{math} \;
        Store trajectory \begin{math}\tau\end{math} in \begin{math} \mathcal{B}\end{math}
    }
    \If { it's time to query instructor} {
        Select queries  $\{ \sigma_i \}^n_{i=1} \leftarrow ActiveSampling(\mathcal B, q_\psi) $ \;
        Collect labels from instruction $\{c_i\}^n_{i=1} $ \;
        Store transitions \begin{math} \zeta^b \leftarrow\zeta^b\cup \{(\sigma_i, c_i)\}^n_{i=1} \end{math}\;
        
        Update \begin{math}q_\psi\end{math} with gradient descent on cross-entropy loss\;    
        }
    Update \begin{math} g_{\phi_{1,2}}\end{math} with gradient ascent on $\mathcal L^{NCE}$ \eqref{eq: nce_loss} \;   
    Update \begin{math}\pi_\theta(a|s,z)\end{math} using TD3 \citep{pmlr-v80-fujimoto18a} and TQC \citep{kuznetsov2020controlling} with \begin{math}r^{\text{SRSD}}\end{math} \eqref{eq: RDSD}\;    
}
\caption{Semantically Relevant Skill Discovery (SRSD)}
\label{alg: SRSD}
\end{algorithm}

%% file: figures/algo_sampling.tex
\begin{algorithm}
\textbf{Input} Replay Buffer $\mathcal B$, Semantic Predictor $q_\psi$ \;
\textbf{Initialise} Number of feedback to collect $n$\;
\textbf{Initialise} Number of segments to sample $l$ \;
Sample segments from replay-buffer: $\{\sigma_i\}^l_{i=0} \sim \mathcal{B}$ \;
Compute pseudo-labels: $\{ c_i \}^{l}_{i=0} \leftarrow \{q_\psi(\sigma_i) \}^{l}_{i=0} $\; 
Split the large batch into buckets according to pseudo-labels: $ \{ \{\sigma_{i,j}\}^{l_j}_{i=0}\}_{j=0}^{|\mathcal C|}$\; 
Remove the bucket with pseudo-label as 'irrelevant': $ \{ \{\sigma_{i,j}\}^{l_j}_{i=0}\}_{j=1}^{|\mathcal C|}$\;  
Select the first $\min (l_j, \frac{n}{|\mathcal C|-1})$-samples for each buckets\;
 \Return $ \{ \{\sigma_{i,j}\}^{\min (l_j, \frac{n}{|C|-1})}_{i=0}\}_{j=1}^{|\mathcal C|}$\;
\caption{Active Sampling}
\label{alg: Active_Sampling}
\end{algorithm}

%% file: appendix/human_feedback.tex
\subsection{Simulating Human Feedback}
\label{subsec: apdx|simulating_human_feedback}
To facilitate experimentation and reproducibility, we simulated human feedback as done in prior work \citep{2021pebble,hussonnois2025human}. Human feedback is modelled by first computing the probability of each relevant semantics being used independently. We then select the semantic with the highest probability and use this to stochastically choose between labelling the segment as the relevant semantic or irrelevant semantically. For each relevant semantics, the probability is a re-scaled sum of normalised rewards over a segment of length $H$: $p(c|\sigma) = \min\left(\frac{\sum_t^H r_t^c}{H \cdot \tau^{c}}, 1\right)$, where $c \in \mathcal{C}^+$ and $\tau^{c}$ is a threshold. This re-scaling accounts for varying reward magnitudes across relevant semantics. We then deterministically select $p(c^+|\sigma)=\arg\max_{c\in \mathcal{C}^+}p(c|\sigma)$ to sample a Bernoulli Distribution, labelling a segment as $c^+$ with probability $p(c^+|\sigma)$ and as irrelevant $c^-$ with probability $1-p(c^+|\sigma)$. This method realistically models human labelling stochasticity, where the likelihood of identifying relevant semantics increases with rewards.

\subsubsection{Simulating Human Irrationality}
Throughout the experiment section, we simulated humans with a ground truth reward and assumed humans followed the stochastic model previously described. 
To assess the sensibility of our methods to human simulation, we previous work~\citep{lee2021b} by testing our methods with various irrationalities. More specifically we evaluate SRSD with a human that can make mistake(there is a probability that the feedback is change), a myopic human (latter part of a segment influences less the feedback) and an amnesic human (latter part of a segment influences more the feedback), we report the detailed performances per task over zero-shot, few-shot and fine-tuning capability in Figures \ref{fig: appendix|downstream_task|zero_shot|tasks|human}, \ref{fig: appendix|downstream_task|few_shot|tasks|human} and \ref{fig: appendix|downstream_task|fine_tuning|tasks|human} for each variation. We also report the aggregate normalised performance over zero-shot, few-shot and fine-tuning in Figure \ref {fig: appendix|dmc|general|human}. We can observe that each irrationality negatively impact the performance of our method. Amnesic and Myopic performed similarly, the difference might be mitigated by the small segment size (25) choose for the experiments. In the same way that amnesic and myopic data are detrimental to our methods, collecting 10 \% of erroneous feedback would only negatively affect SRSD. However with more than 50 \% of erroneous feedback the 0-shot and few-shot capability are significantly degraded. Despite this, all skill sets were improved after fine-tunning.

\begin{figure}[htbp]
    \centering
    \begin{subfigure}[b]{0.32\textwidth}
        \centering
        \includegraphics[width=\textwidth]{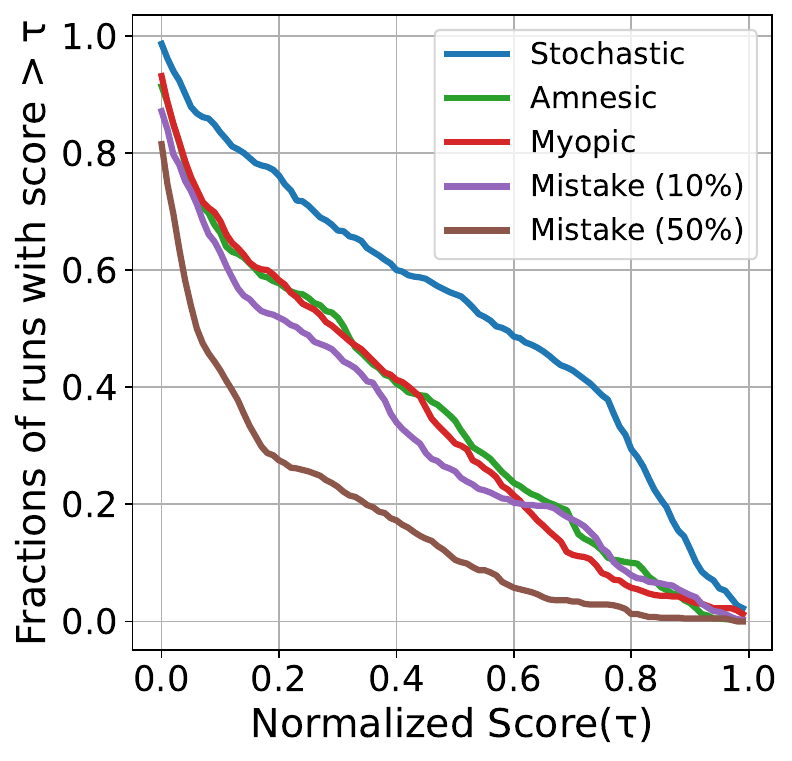} 
        \caption{0-shot}
         \label{fig: appendix|downstream_task|zero_shot|general|human}
    \end{subfigure}
    \hfill
    \begin{subfigure}[b]{0.32\textwidth}
        \centering
        \includegraphics[width=\textwidth]{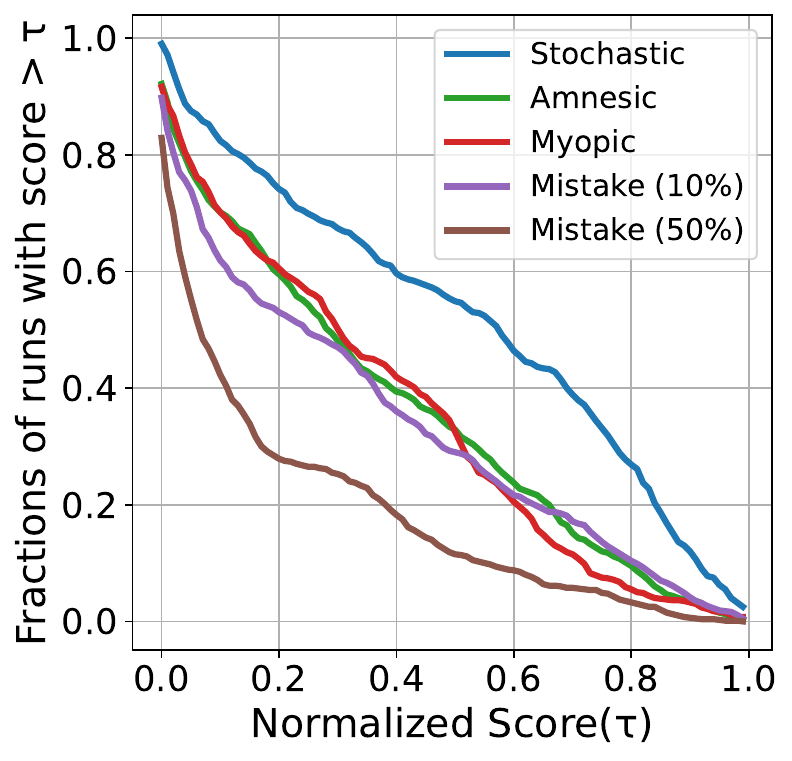} 
        \caption{Few-shot}
         \label{fig: appendix|downstream_task|few_shot|general|human}
    \end{subfigure}
    \hfill
    \begin{subfigure}[b]{0.32\textwidth}
        \centering
        \includegraphics[width=\textwidth]{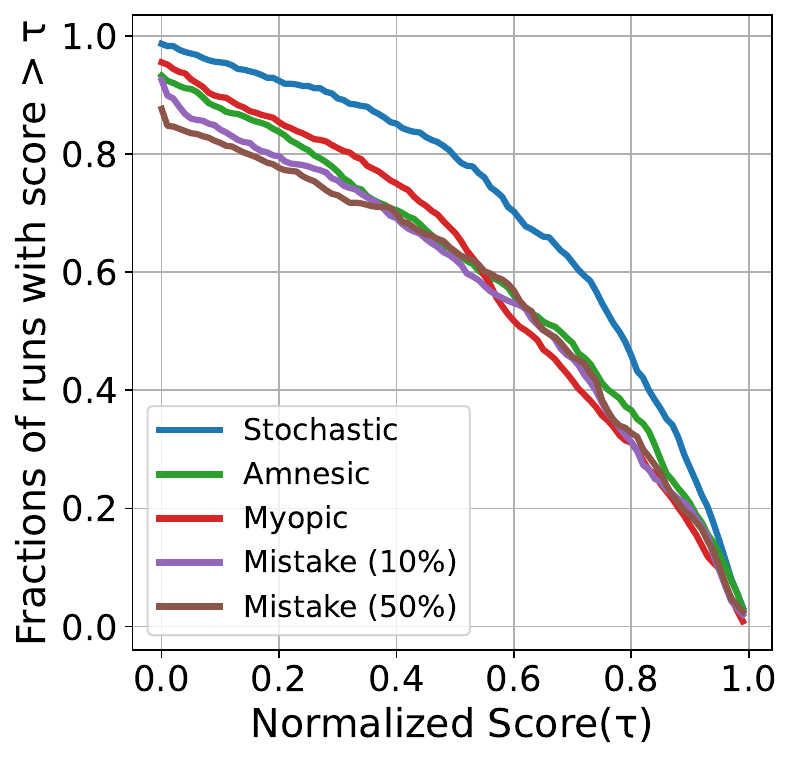} 
        \caption{Fine-Tuning}
        \label{fig: appendix|downstream_task|fine_tuning|general|human}
    \end{subfigure}
    \caption{Aggregated and normalised score distributions for 12 downstream tasks on URLB for various simulated human irrationality.}
    \label{fig: appendix|dmc|general|human}
\end{figure}

\begin{figure}[htbp]
    \centering
    \includegraphics[width=1.0\textwidth]{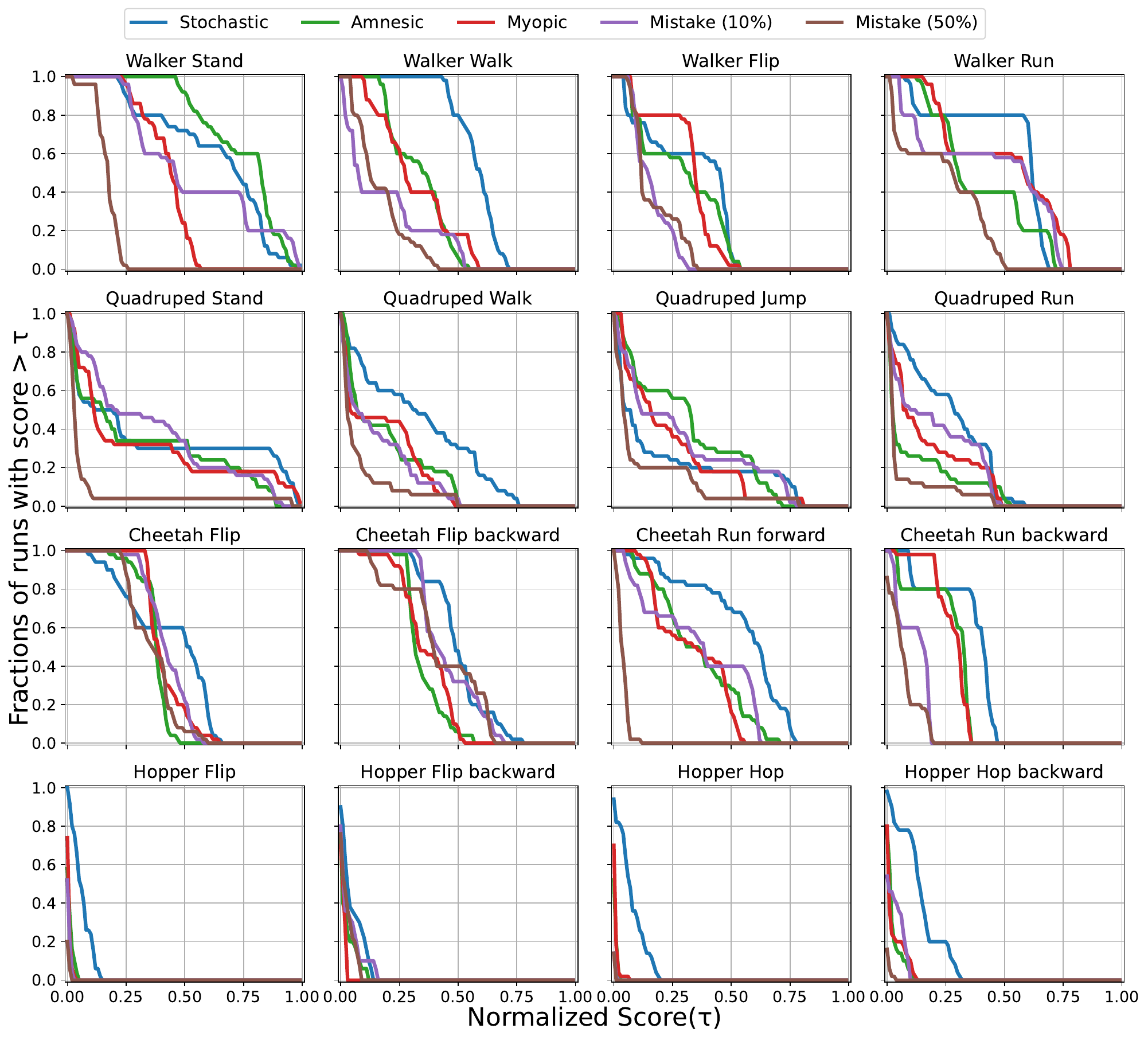} 
    \caption{Score distributions of 12 downstream tasks on URLB over 0-shot evaluation for various simulated human irrationality. Score distributions show the fraction of runs above a certain score (higher curve is better)\cite{agarwal2021deep}}
    \label{fig: appendix|downstream_task|zero_shot|tasks|human}
\end{figure}

\begin{figure}[htbp]
    \centering
    \includegraphics[width=1.0\textwidth]{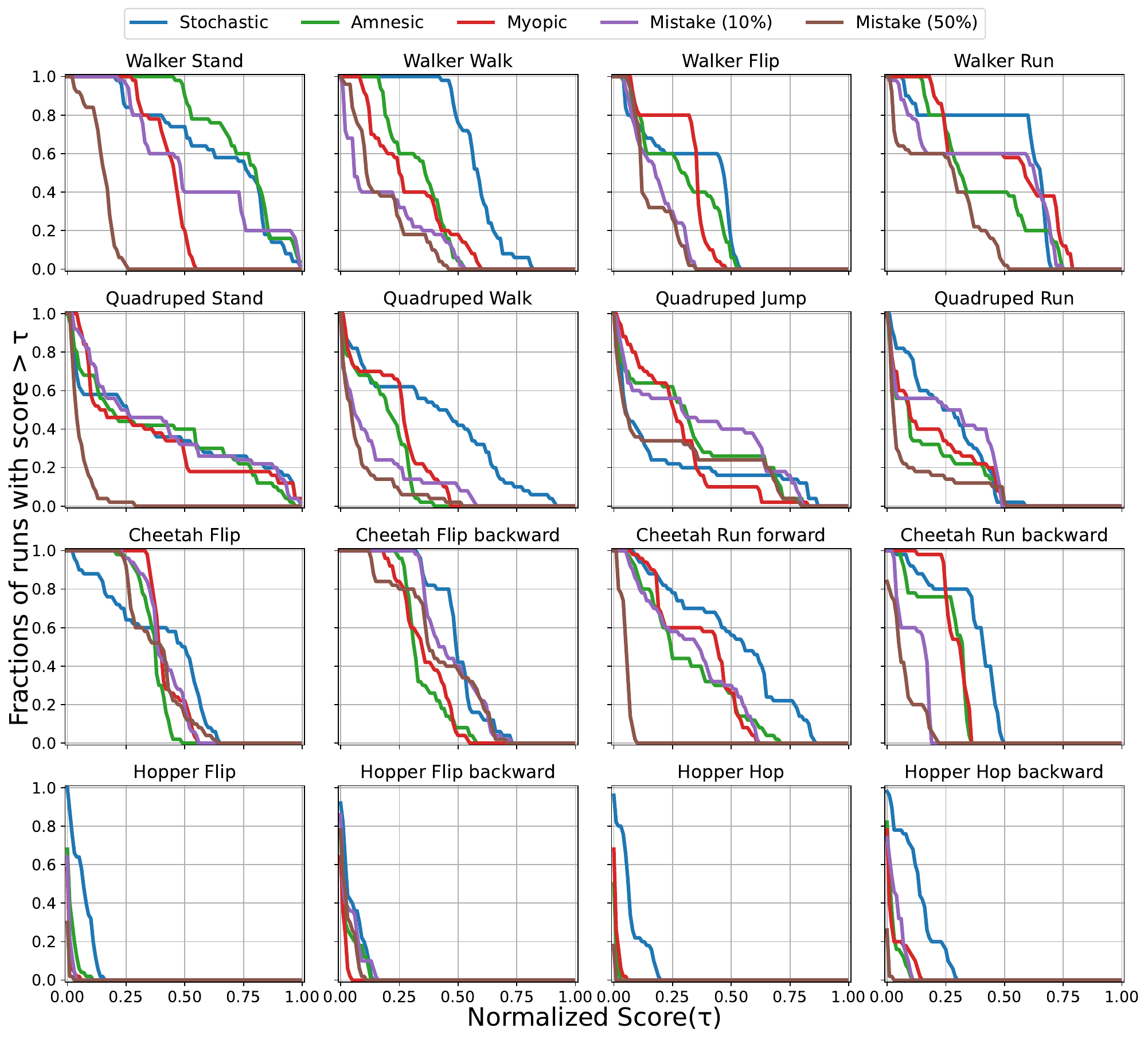} 
    \caption{Score distributions of 12 downstream tasks on URLB over few-shot evaluation for various simulated human irrationality. Score distributions show the fraction of runs above a certain score (higher curve is better)\cite{agarwal2021deep}}
    \label{fig: appendix|downstream_task|few_shot|tasks|human}
\end{figure}

\begin{figure}[htbp]
    \centering
    \includegraphics[width=1.0\textwidth]{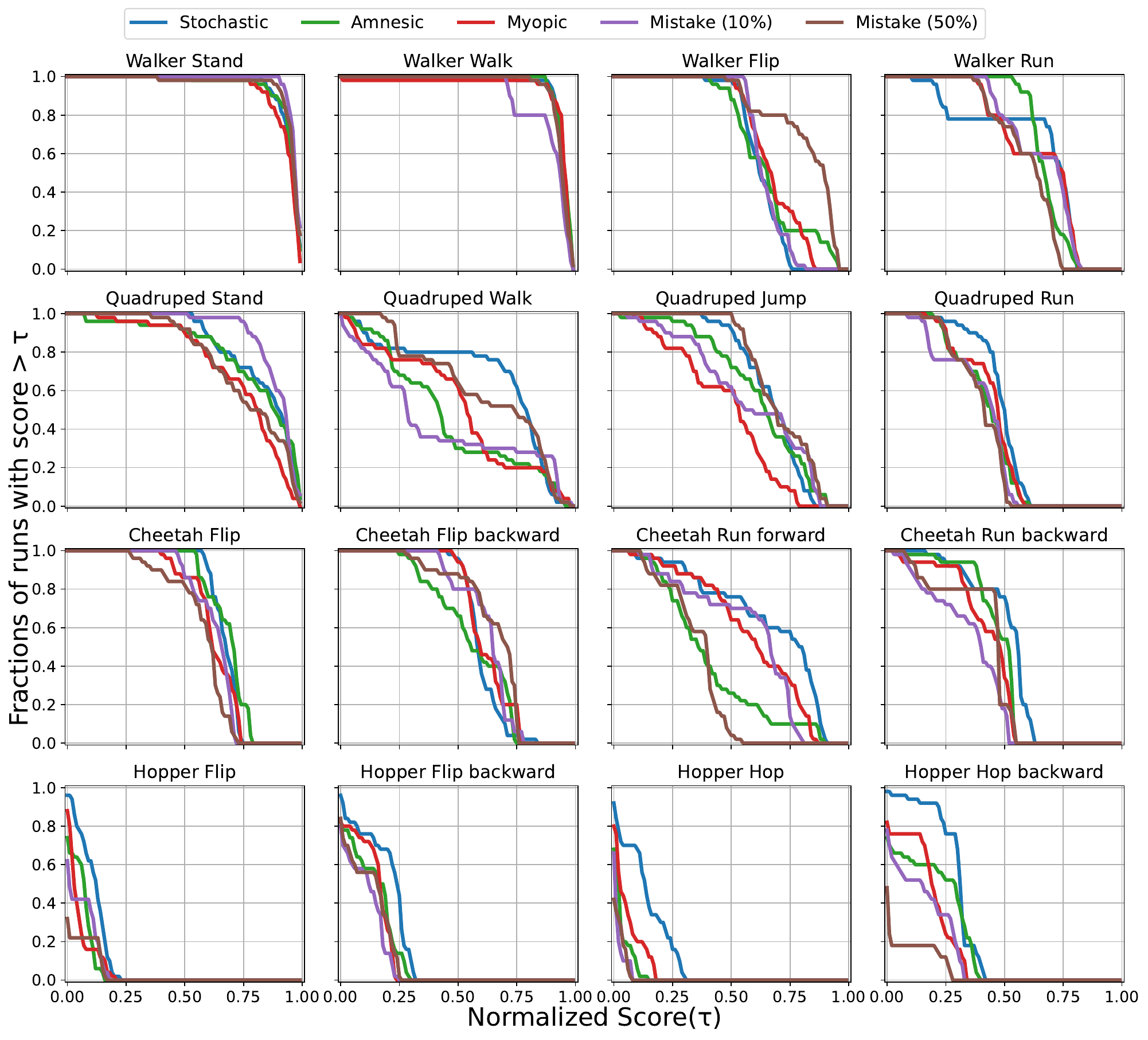} 
    \caption{Score distributions after fine-tuning skills over 12 downstream tasks on URLB for various simulated human irrationality. Score distributions show the fraction of runs above a certain score (higher curve is better)\cite{agarwal2021deep}. SRSD achieves competitive performances across all tasks and agents.}
    \label{fig: appendix|downstream_task|fine_tuning|tasks|human}
\end{figure}